\documentclass{article}

\usepackage[final]{nips_2016}

\usepackage[utf8]{inputenc} 
\usepackage[T1]{fontenc}    
\usepackage{url}            
\usepackage{booktabs}       
\usepackage{amsfonts}       
\usepackage{nicefrac}       
\usepackage{microtype}      

\title{A Modular Theory of Feature Learning}

\usepackage{times}
\usepackage{tikz}
\usepackage{xcolor}
\usepackage[]{algorithm2e}
\usepackage[font={small}]{caption}
\usepackage{amsmath,amssymb}

\usepackage{theorem}
\newcommand{\BlackBox}{\rule{1.5ex}{1.5ex}}  
\newenvironment{proof}{\par\noindent{\bf Proof\ }}{\hfill\BlackBox\\[2mm]}
\newtheorem{theorem}{Theorem}
\newtheorem{lemma}[theorem]{Lemma}

\definecolor{darkspringgreen}{rgb}{0.09, 0.45, 0.27}
\usetikzlibrary{shapes,arrows,calc,patterns,positioning}

\newtheorem{assumption}[theorem]{Assumption}
\author{Daniel McNamara \qquad Cheng Soon Ong \qquad Robert C. Williamson\\
Australian National University and Data61, Canberra ACT 0200, Australia\\
\texttt{\{daniel.mcnamara, chengsoon.ong, bob.williamson\}@anu.edu.au}
}

\begin{document}

\maketitle

\begin{abstract}
Learning representations of data, and in particular learning features for a subsequent prediction task, has been a fruitful area of research delivering impressive empirical results in recent years. However, relatively little is understood about what makes a representation `good'. We propose the idea of a risk gap induced by representation learning for a given prediction context, which measures the difference in the risk of some learner using the learned features as compared to the original inputs. We describe a set of sufficient conditions for unsupervised representation learning to provide a benefit, as measured by this risk gap. These conditions decompose the problem of when representation learning works into its constituent parts, which can be separately evaluated using an unlabeled sample, suitable domain-specific assumptions about the joint distribution, and analysis of the feature learner and subsequent supervised learner. We provide two examples of such conditions in the context of specific properties of the unlabeled distribution, namely when the data lies close to a low-dimensional manifold and when it forms clusters. We compare our approach to a recently proposed analysis of semi-supervised learning.
\end{abstract}

\section{Introduction}
\label{introduction}

The predictive power of machine learning algorithms is highly dependent on the features that they receive as inputs. Traditionally, features have been handcrafted by domain experts. While this works well in some cases, it provides no performance guarantees and requires an expensive custom implementation for each new problem. Representation learning, also known as feature learning, entails automatically transforming input data to enhance the performance of prediction algorithms. In the last decade representation learning techniques using unlabeled data have been used to achieve empirical advances in areas such as computer vision \citep{hinton_reducing_2006} and natural language processing \citep{mikolov_distributed_2013}, and are expected to be at the forefront of further advances in machine learning \citep{lecun_deep_2015}. However, there are few theoretical results on when such techniques offer a benefit.

The main contribution of this paper is a set of sufficient conditions under which unsupervised representation learning provably improves task performance. These conditions can be evaluated using an unlabeled data sample, analysis of the proposed feature learner and supervised learner, and suitable assumptions about shared structure between the marginal distribution $P_X$ and the joint distribution $P_{XY}$. The novelty of our result is its generality beyond any single representation learning technique and its theoretical rather than empirical approach. Furthermore we demonstrate the importance of considering the subsequent task for which the features will be used, including the supervised learner and loss function, in the definition of what makes `good' features.

The paper is structured as follows. In Section \ref{approach} we set out sufficient conditions for unsupervised representation learning to be successful and describe its relationship with semi-supervised learning. In Section \ref{examples}, we instantiate the conditions using
the example of cluster structure in the unlabelled data, with a second example on manifold structure in the Supplement.

\subsection{This is unlike standard learning theory papers}
There are two important features of the paper worth calling out: 1) We analyse a processing \emph{pipeline}, not just a single step. The use of sequential pipelines is common in practice, but rarely addressed theoretically. Our methodology seems novel in this regard. 2) We analyse the problem in terms of risk gaps, rather than sample complexity. This is illustrated in Figure~\ref{flow_chart}, which we explain in detail in Section~\ref{approach}.
While ultimately it is desirable to say something about finite sample performance, the current technology of bounds seems inadequate for the task at hand. Fortunately, by using risk bounds we can legitimately compare performance across the complex pipeline of processing inherent in the problem we address avoiding the impropriety of only comparing upper bounds.

In Section \ref{background} we provide a comparison with current approaches to representation learning as well as existing theoretical results, which are largely focused on the limitations rather than the benefits of representation learning.

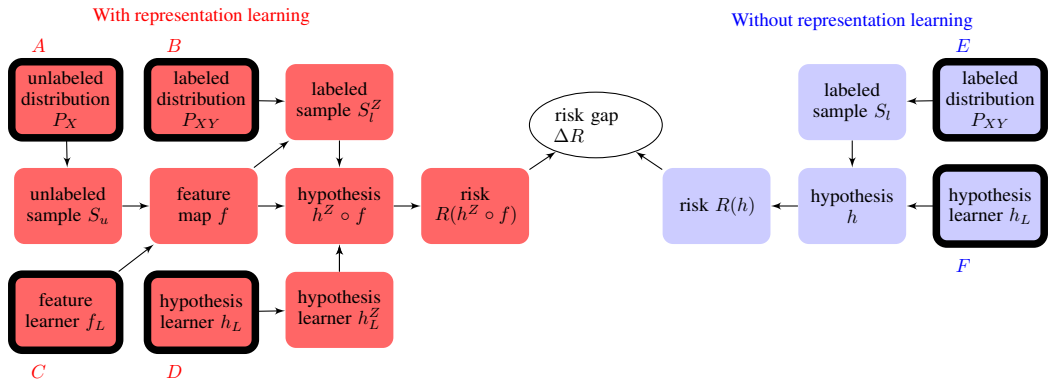
\begin{figure}
\centering

\tikzstyle{a} = [rectangle, draw=none, fill=blue!20,
text width=5em, text centered, rounded corners, minimum height=4em]
\tikzstyle{b} = [rectangle, draw=none, fill=red!60,
text width=5em, text centered, rounded corners, minimum height=4em]
\tikzstyle{line} = [draw, -latex']
\tikzstyle{dashed_line} = [draw, -latex',dashed]
\tikzstyle{cloud} = [draw, ellipse, node distance=2.5cm,
minimum height=2em, text width=1.5cm]
\tikzset{
between/.style args={#1 and #2}{
at = ($(#1)!0.5!(#2)$)
}
}
\begin{tikzpicture}[node distance = 2.5cm, auto,scale=0.72, transform shape]]
\node [cloud] (risk_gap) {risk gap $\Delta R$};
\node [b, below left=0.5cm of risk_gap] (risk_Z) {risk $R(h^Z \circ f)$};
\node [b, left =0.5cm of risk_Z] (hypothesis_Z) {hypothesis $h^Z \circ f$};
\node [a, below right=0.5cm of risk_gap] (risk) {risk $R(h)$};
\node [a, right =0.5cm of risk] (hypothesis) {hypothesis $h$};
\node [a, right =0.5cm of hypothesis,draw=black,line width=1mm,label={[xshift=-0.5cm, yshift=-2.1cm,style=blue]$F$}] (hypothesis_learner) {hypothesis learner $h_L$};
\node [a, above right =0.5cm and 0.5cm of hypothesis,draw=black,line width=1mm,label={[xshift=-0.5cm, yshift=0cm,style=blue]$E$}] (PXY) {labeled distribution $P_{XY}$};
\node [a, above =0.5cm of hypothesis] (sample) {labeled sample $S_l$};
\node [b, below left =0.5cm and 0.5cm of hypothesis_Z,draw=black,line width=1mm, label={[xshift=-0.5cm, yshift=-2.1cm,style=red]$D$}] (hypothesis_learner_top) {hypothesis learner $h_L$};
\node [b, below = 0.5cm of hypothesis_Z] (hypothesis_learner_Z) {hypothesis learner $h_L^Z$};
\node [b, above left =0.5cm and 0.5cm of hypothesis_Z,draw=black,line width=1mm,label={[xshift=-0.5cm, yshift=0cm,style=red]$B$}] (PXY_top) {labeled distribution $P_{XY}$};
\node [b, left =0.5cm of hypothesis_Z] (feature_map) {feature map $f$};
\node [b, above left =0.5cm and 0.5cm of feature_map,draw=black,line width=1mm,label={[xshift=-0.5cm, yshift=0cm,style=red]$A$}] (PX) {unlabeled distribution $P_X$};
\node [b, below left =0.5cm and 0.5cm of feature_map,draw=black,line width=1mm,label={[xshift=-0.5cm, yshift=-2.1cm,style=red]$C$}] (feature_learner) {feature learner $f_L$};
\node [b, above =0.5cm of hypothesis_Z] (sample_Z) {labeled sample $S_l^Z$};
\node [b, left =0.5cm of feature_map] (sample_unlabeled) {unlabeled sample $S_u$};
\path [line] (feature_learner) -- (feature_map);
\path [line] (hypothesis_learner_top) -- (hypothesis_learner_Z);
\path [line] (feature_map) -- (hypothesis_Z);
\path [line] (hypothesis_Z) -- (risk_Z);
\path [line] (risk_Z) -- (risk_gap);
\path [line] (hypothesis_learner_Z) -- (hypothesis_Z);
\path [line] (hypothesis_learner_Z) -- (hypothesis_Z);
\path [line] (risk) -- (risk_gap);
\path [line] (hypothesis) -- (risk);
\path [line] (PXY) -- (sample);
\path [line] (sample) -- (hypothesis);
\path [line] (PXY_top) -- (sample_Z);
\path [line] (feature_map) -- (sample_Z);
\path [line] (sample_Z) -- (hypothesis_Z);
\path [line] (hypothesis_learner) -- (hypothesis);
\path [line] (PX) -- (sample_unlabeled);
\path [line] (sample_unlabeled) -- (feature_map);
\node [above =0.5cm of PXY_top,red](feature_learning) {With representation learning};
\node [above =0.5cm of sample,blue](no_feature_learning) {Without representation learning};
\end{tikzpicture}
\caption{Measuring the effect of unsupervised representation learning. The red path (left) shows the approach with representation learning, the blue path (right) shows the approach without representation learning, and the risk gap determines which of the two paths has lower risk. The arrows indicate dependencies. Source nodes are shown with a black border and are annotated with corresponding conditions from Table \ref{conditions_table}.} \label{flow_chart}
\end{figure}

\section{Related work}
\label{background}

Many representation learning techniques have been developed, including those using unlabeled data, and have achieved considerable empirical success \citep{bengio_representation_2013} but few theoretical guarantees concerning the effect on task performance.
The literature to date demonstrates the usefulness of such techniques while also highlighting the need for more analysis of when and why they work.

The desire for computational efficiency has motivated techniques to learn low-dimensional manifold embeddings. Principal components analysis (PCA) is the canonical such technique, and has been extended to kernel variants such as Isomap, Laplacian eigenmaps and local linear embedding \citep{mohri_foundations_2012}. It has been shown theoretically that it is possible to compress a finite set of high dimensional points to a low dimensional representation while bounding the distortion in pointwise Euclidean distances \citep{johnson_extensions_1984}. However, manifold learning approaches have typically not proved any improvement in the performance of a subsequent learner.

Empirical results in the field of deep learning have shown the power of learning multiple levels of representations. While more recent results have focused on supervised representation learning, initial advances used unsupervised techniques such as the autoencoder \citep{hinton_reducing_2006}. The effect of unsupervised pre-training has been studied empirically \citep{erhan_why_2010}, with observed benefits in terms of both reduced training set error and improved generalization. While attempts have been made to theorize unsupervised representation learning in studies such as \citep{saxe_exact_2014} --- which concluded that a certain kind of random initialization could achieve the same condition as unsupervised pre-training  --- mostly experimental results have outpaced theory. Such techniques often learn overcomplete representations (in a higher dimension than the original inputs), moving away from a paradigm of dimensionality reduction to one of learning features which are well-suited to the final classifier, which in neural networks is typically a linear separator.

Despite these advances, theoretical results derived from information theory are pessimistic. Using learned features can never decrease the risk of the Bayes optimal classifier \citep{van_rooyen_theory_2015}. This is because the set of hypotheses involving a feature transformation step followed by a prediction step is a subset of all possible hypotheses. This result is similar to the data processing inequality, which states that if random variables $x,y,z$ form a Markov chain $y \to x \to z$, then $I(z,y) \leq I(x,y)$ where $I$ is mutual information \citep{coveR_Elements_2012}. Both results share the idea that information cannot be created by data manipulation. Hence, representation learning cannot be shown to be useful without considering the subsequent hypothesis learner.

\subsection{Relationship to existing representation learning approaches}
\label{existing_algorithms_relationship}

Techniques aimed at learning linearly separable features can also be analyzed using our theorems. Our example (see Section \ref{cluster}), which involves exploiting cluster structure in the unlabeled data, shares the key objective of deep neural networks: learning a representation which will be linearly separable. Methods which aim to learn the kernel also share this objective, although the kernel function implicitly rather than explicitly defines the feature space.

Manifold learning techniques can be analyzed using our theorems, and is discussed in the Supplement.
The conditions in these cases require that the unlabeled data lies near a low-dimensional manifold, that the manifold structure is related to the labels, and that the proposed reparametrization in terms of the manifold structure is compatible with the supervised learner used.

A distinctive aspect of our approach is that it provides high probability performance guarantees in terms of the performance of a particular subsequent learner, unlike most unsupervised learning approaches where the objective function has no relation to the supervised task of interest \citep{sutskever_towards_2015}. The high probability risk bounds we present in our examples demonstrate the feasibility of our approach but in future could be tightened. Another original element is the use of property testing to determine which representation learning technique is best suited to the unlabeled data, rather than adopting a `one size fits all' approach.

\subsection{Relationship to semi-supervised learning}
\label{semisupervised_relationship}

The schematic presented in Figure \ref{flow_chart} can be seen as a special case of a more general dilemma faced by machine learning practitioners. Given some existing system, an additional step is proposed. We would like to know under what circumstances such a step will enhance the system's performance. We may also characterize semi-supervised learning in this way, where the step is some particular use of unlabeled data. The differences between the two approaches are shown in Figure \ref{ssl_fig}.

\begin{figure}
\centering

\begin{tikzpicture}[scale=0.7, transform shape]

\path [draw=none,fill=blue!30,opacity=0.5] (0,0) circle (2);
\path [draw=none,fill=red,opacity=0.5] (0,0) circle (1);
\node[draw=none] at (0,2.8) {Semi-supervised learning};
\node[draw=none] at (6.2,2.8) {Unsupervised representation learning};
\node[draw=none] at (0,2.2) {$H$};
\node[draw=none] at (0,0) {$H_{P_X,\chi}(\epsilon)$};
\path [draw=none,fill=black] (0,-0.8) circle (0.1);
\node [left] at (0.15,-0.7) {$h^*$};
\node[draw=none] at (5.5,2.2) {$H$};

\path [draw=none,fill=blue!30,opacity=0.5] (5.5,0) circle (2);
\path [draw=none,fill=red,opacity=0.5] (7.5,0) circle (1.5);
\node[draw=none] at (7.5,1.7) {$H^Z \circ f$};
\path [draw=none,fill=black] (8.5,-0.8) circle (0.1);
\node [left] at (8.65,-0.7) {$h^*$};
\end{tikzpicture}

\caption{Relationship between semi-supervised learning (left) and unsupervised representation learning (right). In the formulation of semi-supervised learning proposed by \cite{balcan_discriminative_2010}, the hypothesis space $H$ is pruned to a subset $H_{P_X,\chi}(\epsilon)\subset H$ which contains the target function $h^*$ (see Appendix \ref{ssl_bound}). In unsupervised representation learning, the hypothesis space changes to $H^Z \circ f$, which is especially useful when $h^* \in H^Z \circ f \wedge h^* \not\in H$. The feature map $f$ is learned from unlabeled data. Cases where $H^Z$ is related to $H$ through transparent polymorphism are of particular interest.} \label{ssl_fig}
\vspace{-15pt}
\end{figure}
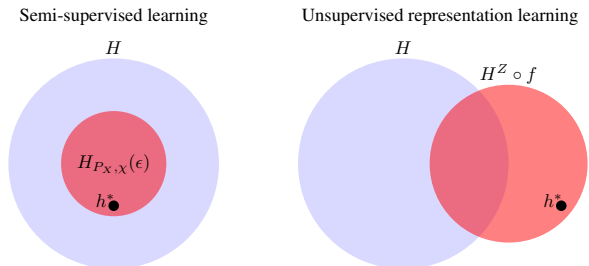

In semi-supervised learning, if a compatibility function exists which allows the elimination of incompatible hypotheses using only unlabeled data, performance may improve as the optimization will be more straightforward over a smaller hypothesis class and less sensitive to noise where few labeled data points are available. Furthermore, a tighter generalization error bound will be possible \citep{balcan_discriminative_2010} (see Appendix \ref{ssl_bound}). However, if the target function lies outside the original hypothesis class, semi-supervised learning will not help to discover it.

In unsupervised representation learning, the hypothesis class changes and hence it is possible to learn hypotheses not included in the original hypothesis class. The cluster example (Section \ref{cluster}) illustrates this point. In some cases the size of the hypothesis class is reduced, such as using a lower-dimensional representation, so that a tighter generalization error bound is possible, or equivalently a reduction in sample complexity (see Theorem \ref{sample_complexity_theorem}). This paper is more ambitious again in that it seeks to show that the risk upper bound using the learned representation is not only lower than the original risk upper bound, but lower than the original risk \textit{lower} bound (see Theorem \ref{risk_gap_theorem}).

\section{When unsupervised representation learning improves task performance}
\label{approach}

Our goal is to determine when unsupervised representation learning enhances the performance of a subsequent supervised hypothesis learner. This situation describes a range of common machine learning scenarios. Do the features learned by an autoencoder enhance the performance of a linear classifier compared to using the original inputs? Does unsupervised pre-training improve the performance of a supervised neural network? Does a particular kernel function outperform a linear kernel when used with a hypothesis class of linear separators (recalling that kernel functions implicitly specify a feature space)? Do distributed vector representations of words outperform one-hot unigram representations for natural language processing tasks? Our approach to estimating the effect of unsupervised representation learning is shown in Figure \ref{flow_chart}. We examine when it can be shown that the path including a representation learning step reduces risk.

\subsection{Problem statement}
\label{problem_statement}

Let $X$, $Y$ and $Z$ be the input, output and learned feature spaces respectively. Let $f_L$ be an unsupervised feature learning algorithm, $S_u \in X^{m_u}$ be an unlabeled sample, and $f: X \to Z$ be the feature map learned using $f_L$ and $S_u$. Let $h_L$ be a supervised hypothesis learner, $S_l \in \{X \times Y\}^{m_l}$ be a labeled sample, and $h: X \to Y$ be the hypothesis learned using $h_L$ and $S_l$. Let $h_L^Z$ be the supervised hypothesis learner accepting inputs in the learned feature space, $S_l^Z=\bigcup\limits_{\{x,y\}\in S_l}\{f(x),y\}$ be the labeled sample transformed into the learned feature space, and $h^Z: Z \to Y$ be the hypothesis learned using $h_L^Z$ and $S_l^Z$. The risk of $h$ is $R(h)=\mathbb{E}_{\{x,y\} \sim P_{XY}}[L(h(x),y)]$, where $L$ is a loss function. Similarly, let the risk of $h^Z$ using the feature map $f$ be $R(h^Z \circ f)=\mathbb{E}_{\{x,y\} \sim P_{XY}}[L(h^Z(f(x)),y)]$.

The procedure in Figure \ref{flow_chart} involves comparing a learner $h_L$, which produces an hypothesis of type $X \to Y$, with a learner $h_L^Z$ which produces an hypothesis of type $Z \to Y$. While these learners have different type signatures, to isolate the effect of representation learning it will be convenient to consider situations where the two learners are similar; for example, both are logistic regression but accept inputs of different dimension. If it is possible to construct $h_L^Z$ from $h_L$ through a straightforward change of type signature without substantively changing the learner, we will refer to $h_L$ as `transparently polymorphic' (see Supplement \ref{transparent_polymorphism} for further discussion).

\begin{table}
\small

\begin{center}
\begin{tabular}{p{1.5cm}|p{3.8cm}|p{4.8cm}|p{2.2cm}}

\bf Condition & \bf Description & \bf Formal statement & \bf Verification \\ \hline
\multicolumn{4}{l}{\it Upper bound on risk using representation learning} \\ \hline
\rule{0pt}{3ex}$A(P_X)$ & If property test passes, marginal distribution has some structure & $\hat{R}_a(P_X)\leq\hat{\epsilon}_A \implies$  \newline $R_a(P_X)\leq\epsilon_A \wedge R_A(P_X)\leq\epsilon_A$ & Analysis of property test \\ \hline
$B(P_{XY})$ & Joint distribution shares marginal distribution structure if present & $R_A(P_X)\leq\epsilon_A \implies$ $R_B(P_{XY})\leq\epsilon_B$ & Domain expert \\ \hline
\rule{0pt}{3ex}$C(f_L)$ &  Feature learner can exploit marginal distribution structure if present & $ R_a(P_X)\leq\epsilon_A \implies R_C(f)\leq \epsilon_C$ & Analysis of $f_L$  \\ \hline
\rule{0pt}{3ex}$D(h_L)$ & Hypothesis learner can exploit learned features & $R_B(P_{XY})\leq\epsilon_B \wedge R_C(f)\leq\epsilon_C \implies R(h^Z \circ f)\leq \epsilon_{\max}^Z$ & Analysis of $h_L$ \\ \hline
\multicolumn{4}{l}{\it Lower bound on risk without using representation learning} \\ \hline
$E(P_{XY})$ & Joint distribution has property that it is `hard' to learn & $R_E(P_{XY})\geq\epsilon_E$ & Domain expert \\ \hline
$F(P_X,h_L)$ & Hypothesis learner cannot learn accurately from original inputs & $\int_yP_{XY}dy=P_X \wedge R_E(P_{XY})\geq\epsilon_E\wedge R_B(P_{XY})\leq\epsilon_B \implies R(h)\geq \epsilon_{\min}$ & Test using unlabeled sample and analysis of $h_L$ \\
\end{tabular}
\end{center}
\caption{Sufficient conditions for representation learning to improve performance. The definitions in the `Formal statement' column use a series of intermediate risk terms as discussed in Section \ref{sufficient_conditions}, and are defined for our two working examples in Tables \ref{manifold_table} and \ref{cluster_table} (see Supplement). The `Verification' column indicates the approach to checking whether the condition holds. See Figure \ref{dependency_map} for a visualization of the conditions.} \label{conditions_table}
\end{table}

\subsection{Set of sufficient conditions}
\label{sufficient_conditions}

We provide a set of sufficient conditions for unsupervised representation learning to reduce the risk of a hypothesis learner, as shown in Table \ref{conditions_table}. We motivate these conditions by noting that each independent aspect of the prediction context (the source nodes in Figure \ref{flow_chart}) is associated with a condition. Hence we are unlikely to be able to reduce the set of conditions, a topic we discuss further in Supplement \ref{motivation}. The conditions allow us to measure the effect of representation learning by separately analyzing specific aspects of the problem setting and then combining the results.

The conditions use a series of intermediate risk terms, measuring the extent to which particular properties of the prediction context hold. $R_A(P_X)$ measures the extent to which some member of a class of structural properties holds for $P_X$, $R_a(P_X)$ measures the extent to which some specific structural property holds, and $\hat{R}_a(P_X)$ is an empirical estimate of $R_a(P_X)$ calculated using the unlabeled sample $S_u$ and a property test. Furthermore, $R_B(P_{XY})$ measures the extent to which the labeled distribution $P_{XY}$ shares structure with the unlabeled distribution, $R_C(f)$ measures the extent to which the learned feature map $f$ exploits the structure of the unlabeled distribution, and $R_E(P_{XY})$ measures the complexity of the joint distribution.

We show that, given an unlabeled sample and a specific prediction context, it is possible to obtain a high probability upper bound on the risk of a hypothesis learner using features learned from the unlabeled sample, and a high probability lower bound on the risk gap induced by using these learned features. A standard approach would be to compare the results of experiments on a supervised task with and without the representation learning step. Our alternative approach offers two benefits. First, it provides a guarantee on the benefit of representation learning for a class of tasks rather than requiring validation for each task individually. Second, it disaggregates why representation learning is effective, allowing the development of new techniques that are theoretically well-grounded.

 Theorem \ref{sample_complexity_theorem} shows that if a set of sufficient conditions hold, with high probability it is possible to upper bound the risk of a supervised learner using features learned from unlabeled data. The conditions can each be separately evaluated and collectively guarantee that the bound holds. The conditions require that $P_X$ has some testable structure, that $P_{XY}$ shares structure with $P_X$, that the feature learner $f_L$ exploits the structure in $P_X$, and that the hypothesis learner $h_L$ exploits the learned features. The bound $\epsilon_{\max}^Z$ achieved is specific to the prediction context (see Theorem \ref{manifold_theorem} for an example). Hence the contribution of the theorem is the decomposition of the problem into tractable components, rather than a particular numerical bound.  The result motivates a high level algorithm to select the best feature learner $f_L$ from a number of options, as shown in Algorithm \ref{feature_learning_Algorithm} (see Supplement \ref{algorithm_uor_selecting_A_feature_learner}). Such risk bounds can also be used to compute labeled sample complexity.

\begin{theorem}
\label{sample_complexity_theorem}
Suppose that given a sample $S_u$ drawn from $P_X$ and a property test, $\hat{R}_a(P_X)\leq\hat{\epsilon}_A$. Suppose also that given a feature learner $f_L$ and a hypothesis learner $h_L$, it can be shown that with probability at least $1-\delta$, conditions $A, B, C$ and $D$ in Table \ref{conditions_table} all hold. Then if a hypothesis $h^Z \circ f$ is constructed from $S_u$ and a sample $S_l$ drawn from $P_{XY}$ as described in Section \ref{problem_statement}, with probability at least $1-\delta$, $R(h^Z \circ f) \leq \epsilon_{\max}^Z$.
\end{theorem}

Theorem \ref{risk_gap_theorem} shows that if additional conditions hold, with high probability unsupervised representation learning reduces risk compared to using the original inputs for some supervised learner. These additional conditions require that the joint distribution is `hard' to learn and that the hypothesis learner cannot learn accurately from the original inputs. The result will be of interest when the risk gap $\Delta R>0$. This is the main result of the paper as it decomposes the effect of unsupervised representation learning into a set of conditions that can be separately evaluated using an unlabeled sample, suitable domain-specific assumptions about $P_{XY}$, and analytical properties of the feature learner $f_L$ and hypothesis learner $h_L$. Once again, the bound $\epsilon_{\min}-\epsilon_{\max}^Z$ achieved is specific to the prediction context (see Theorem \ref{cluster_theorem} for an example), rather than a particular numerical bound.

\begin{theorem}
\label{risk_gap_theorem}
Suppose that given a sample $S_u$ drawn from $P_X$ and a property test, $\hat{R}_a(P_X)\leq\hat{\epsilon}_A$. Suppose also that given a feature learner $f_L$ and a hypothesis learner $h_L$, it can be shown with probability at least $1-\delta$, conditions $A, B, C, D, E$ and $F$ in Table \ref{conditions_table} all hold. Then if hypotheses $h$ and $h^Z \circ f$ are constructed from $S_u$ and a sample $S_l$ drawn from $P_{XY}$ as described in Section \ref{problem_statement}, with probability at least $1-\delta$, $\Delta R:=R(h)-R(h^Z \circ f) \geq \epsilon_{\min}-\epsilon_{\max}^Z$.
\end{theorem}

The proofs of the above theorems essentially just string together the conditions.
We instantiate Theorems \ref{sample_complexity_theorem} and \ref{risk_gap_theorem} with two examples, in Section \ref{examples} and the Supplement.

\section{Application of Theorem~\ref{sample_complexity_theorem} and \ref{risk_gap_theorem}: Cluster representation}
\label{examples}

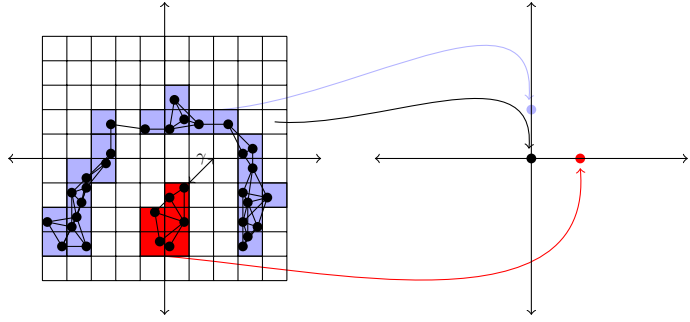
\begin{figure}[t]
\centering

\begin{tikzpicture}[scale=0.65, transform shape]
\draw[->,blue!30] (0.5,1) to[out=-5,in=85] (7.45,1.2);
\draw[->,red] (0,-2) to[out=-5,in=-85] (8.5,-0.2);
\draw[->,black] (2.25,0.75) to[out=-5,in=85] (7.45,0.2);
\draw [<->] (-3.2,0) -- (3.2,0);
\draw [<->] (0,-3.2) -- (0,3.2);
\draw [<->](0.5,-0.5) -- node[above] {$\gamma$} ++(0.5,0.5);
\path [draw=none,fill=blue!30] (7.5,1) circle (0.1);
\path [draw=none,fill=red] (8.5,0) circle (0.1);
\path [draw=none,fill=black] (7.5,0) circle (0.1);
\draw [<->] (4.3,0) -- (10.7,0);
\draw [<->] (7.5,-3.2) -- (7.5,3.2);

\draw [fill=red]  (-0.5,-2) rectangle (0,-1);
\draw [fill=red]  (0,-2) rectangle (0.5,-0.5);

\draw [fill=blue!30]  (-2.5,-2) rectangle (-2,-1);
\draw [fill=blue!30]  (-2,-2) rectangle (-1.5,0);
\draw [fill=blue!30]  (-1.5,-0.5) rectangle (-1,0.5);
\draw [fill=blue!30]  (-1.5,0.5) rectangle (-1,1);
\draw [fill=blue!30]  (0,0.5) rectangle (0.5,1.5);
\draw [fill=blue!30]  (0.5,0.5) rectangle (1,1);
\draw [fill=blue!30]  (1.5,0) rectangle (2,0.5);
\draw [fill=blue!30]  (1.5,-2) rectangle (2,0);
\draw [fill=blue!30]  (2,-1) rectangle (2.5,-0.5);
\draw [fill=blue!30]  (-0.5,0.5) rectangle (0,1);
\draw [fill=blue!30]  (1,0.5) rectangle (1.5,1);

\draw[step=0.5,black,thin] (-2.5,-2.5) grid (2.5,2.5);

\path [draw=none,fill=black] (1.7,-0.9) circle (0.1);
\draw [-] (1.7,-0.9) -- (1.6,-1.3);
\path [draw=none,fill=black] (1.6,-1.3) circle (0.1);
\draw [-] (1.6,-1.8) -- (1.6,-1.3);
\path [draw=none,fill=black] (1.6,-1.8) circle (0.1);
\draw [-] (1.6,-1.8) -- (1.7,-1.6);
\draw [-] (1.6,-1.3) -- (1.7,-1.6);
\path [draw=none,fill=black] (1.7,-1.6) circle (0.1);
\draw [-] (1.9,-1.4) -- (1.7,-1.6);
\draw [-] (1.9,-1.4) -- (1.6,-1.8);
\draw [-] (1.9,-1.4) -- (1.6,-1.3);
\draw [-] (1.9,-1.4) -- (1.7,-0.9);
\draw [-] (1.7,-0.9) -- (1.7,-1.6);
\path [draw=none,fill=black] (1.9,-1.4) circle (0.1);
\draw [-] (1.7,-0.9) -- (2.1,-0.8);
\draw [-] (1.9,-1.4) -- (2.1,-0.8);
\draw [-] (1.6,-1.3) -- (2.1,-0.8);
\path [draw=none,fill=black] (2.1,-0.8) circle (0.1);
\draw [-] (1.7,-0.9) -- (1.6,-0.7);
\draw [-] (1.6,-1.3) -- (1.6,-0.7);
\draw [-] (2.1,-0.8) -- (1.6,-0.7);
\path [draw=none,fill=black] (1.6,-0.7) circle (0.1);
\draw [-] (2.1,-0.8) -- (1.8,-0.2);
\draw [-] (1.7,-0.9) -- (1.8,-0.2);
\draw [-] (1.6,-0.7) -- (1.8,-0.2);
\path [draw=none,fill=black] (1.8,-0.2) circle (0.1);
\draw [-] (1.8,-0.2) -- (1.6,0.1);
\path [draw=none,fill=black] (1.6,0.1) circle (0.1);
\draw [-] (1.6,0.1) -- (1.8,0.2);
\draw [-] (1.8,-0.2) -- (1.8,0.2);
\path [draw=none,fill=black] (1.8,0.2) circle (0.1);
\draw [-] (1.3,0.7) -- (1.8,0.2);
\draw [-] (1.3,0.7) -- (1.6,0.1);
\path [draw=none,fill=black] (1.3,0.7) circle (0.1);
\draw [-] (1.3,0.7) -- (0.7,0.7);
\path [draw=none,fill=black] (0.7,0.7) circle (0.1);
\draw [-] (0.7,0.7) -- (0.4,0.8);
\path [draw=none,fill=black] (0.4,0.8) circle (0.1);
\draw [-] (0.4,0.8) -- (0.2,1.2);
\draw [-] (0.7,0.7) -- (0.2,1.2);
\path [draw=none,fill=black] (0.2,1.2) circle (0.1);
\draw [-] (0.7,0.7) -- (0.1,0.6);
\draw [-] (0.4,0.8) -- (0.1,0.6);
\draw [-] (0.1,0.6) -- (0.2,1.2);
\path [draw=none,fill=black] (0.1,0.6) circle (0.1);
\draw [-] (0.1,0.6) -- (-0.4,0.6);
\path [draw=none,fill=black] (-0.4,0.6) circle (0.1);
\draw [-] (-1.1,0.7) -- (-0.4,0.6);
\path [draw=none,fill=black] (-1.1,0.7) circle (0.1);
\draw [-] (-1.1,0.7) -- (-1.1,0.1);
\path [draw=none,fill=black] (-1.1,0.1) circle (0.1);
\draw [-] (-1.1,0.1) -- (-1.2,-0.1);
\path [draw=none,fill=black] (-1.2,-0.1) circle (0.1);
\draw [-] (-1.1,0.1) -- (-1.6,-0.4);
\draw [-] (-1.2,-0.1) -- (-1.6,-0.4);
\path [draw=none,fill=black] (-1.6,-0.4) circle (0.1);
\draw [-] (-1.6,-0.4) -- (-1.6,-0.6);
\draw [-] (-1.2,-0.1) -- (-1.6,-0.6);
\path [draw=none,fill=black] (-1.6,-0.6) circle (0.1);
\draw [-] (-1.6,-0.6) -- (-1.7,-0.9);
\draw [-] (-1.6,-0.4) -- (-1.7,-0.9);
\path [draw=none,fill=black] (-1.7,-0.9) circle (0.1);
\draw [-] (-1.7,-0.9) -- (-1.8,-1.2);
\draw [-] (-1.6,-0.6) -- (-1.8,-1.2);
\path [draw=none,fill=black] (-1.8,-1.2) circle (0.1);
\draw [-] (-1.8,-1.2) -- (-1.9,-0.7);
\draw [-] (-1.6,-0.6) -- (-1.9,-0.7);
\draw [-] (-1.7,-0.9) -- (-1.9,-0.7);
\draw [-] (-1.6,-0.4) -- (-1.9,-0.7);
\path [draw=none,fill=black] (-1.9,-0.7) circle (0.1);
\draw [-] (-1.9,-0.7) -- (-1.9,-1.4);
\draw [-] (-1.8,-1.2) -- (-1.9,-1.4);
\draw [-] (-1.7,-0.9) -- (-1.9,-1.4);
\path [draw=none,fill=black] (-1.9,-1.4) circle (0.1);
\draw [-] (-1.9,-1.4) -- (-2.4,-1.3);
\draw [-] (-1.8,-1.2) -- (-2.4,-1.3);
\path [draw=none,fill=black] (-2.4,-1.3) circle (0.1);
\draw [-] (-2.4,-1.3) -- (-2.1,-1.8);
\draw [-] (-1.9,-1.4) -- (-2.1,-1.8);
\draw [-] (-1.8,-1.2) -- (-2.1,-1.8);
\path [draw=none,fill=black] (-2.1,-1.8) circle (0.1);
\draw [-] (-1.8,-1.2) -- (-1.6,-1.8);
\draw [-] (-1.9,-1.4) -- (-1.6,-1.8);
\draw [-] (-2.1,-1.8) -- (-1.6,-1.8);
\path [draw=none,fill=black] (-1.6,-1.8) circle (0.1);

\path [draw=none,fill=black] (0.1,-0.8) circle (0.1);
\draw [-] (0.1,-0.8) -- (0.4,-0.6);
\path [draw=none,fill=black] (0.4,-0.6) circle (0.1);
\draw [-] (0.1,-0.8) -- (0.4,-1.3);
\draw [-] (0.4,-0.6) -- (0.4,-1.3);
\path [draw=none,fill=black] (0.4,-1.3) circle (0.1);
\draw [-] (0.4,-1.3) -- (0.1,-1.8);
\path [draw=none,fill=black] (0.1,-1.8) circle (0.1);
\draw [-] (0.4,-1.3) -- (-0.1,-1.7);
\draw [-] (0.1,-1.8) -- (-0.1,-1.7);
\path [draw=none,fill=black] (-0.1,-1.7) circle (0.1);
\draw [-] (0.4,-1.3) -- (-0.2,-1.1);
\draw [-] (-0.1,-1.7) -- (-0.2,-1.1);
\draw [-] (0.1,-0.8) -- (-0.2,-1.1);
\path [draw=none,fill=black] (-0.2,-1.1) circle (0.1);

\end{tikzpicture}
\caption{Example of a map $f$ (arrows) from the original input space $X=\mathbb{R}^2$ (left) to the feature space $Z=\mathbb{R}^2$ (right). The data lies in $k=2$ clusters separated by margin greater than $\gamma$. The graph constructed from the unlabeled sample is shown, along with the regions formed by the union of the hypercubes associated with the points within each graph component.} \label{clusteR_Assumption_fig}
\vspace{-15pt}
\end{figure}

We present an illustrative example of the sufficient conditions for unsupervised representation learning. The first learns a representation from cluster structure and demonstrates the application of Theorem~\ref{risk_gap_theorem}. The second example in the supplement discusses manifold learning, and applies Theorem~\ref{sample_complexity_theorem}. In both examples we assume a bounded continuous input space $X=[0,1]^n$, a binary output space $Y=\{0,1\}$, and a zero/one loss function $L(y,y')=\mathbf{1}(y\neq y')$.

We sketch the strategy used for our proofs. We prove each condition individually with high probability and then take a union bound to show that with high probability all conditions hold (see Supplement \ref{high_probability_interpretation} for a discussion of these high probability statements). We divide the input space into hypercubes and run an algorithm to test finitely many properties within some property class, yielding the property test result $\hat{R}_a(P_X)$. Condition $A(P_X)$ is achieved by reducing the measurement of a property of $P_X$ to binary classification with a finite hypothesis class, allowing the use of standard finite hypothesis class generalization error bounds. Condition $B(P_{XY})$ is assumed, where a domain expert specifies shared structure between $P_X$ and $P_{XY}$ if the property test passes. Condition $C(f_L)$ is derived by designing the property test such that it checks that the feature learner $f_L$ will work on $P_X$.  Condition $D(h_L)$ is shown using the fact that, with high probability, if a large enough labeled sample is drawn from a finite number of bins, the total probability mass of bins containing no labeled points will not be too large (see Lemma \ref{alpha_lemma}). In the second example, condition $E(P_{XY})$ on the complexity of the joint distribution is also provided by a domain expert. Condition $F(P_X,h_L)$ is shown by adding the most favorable possible labels to the unlabeled sample $S_u$, running $h_L$ on this training set to determine the minimum empirical risk achievable by a hypothesis learned by $h_L$, and then using a standard VC dimension-based result to lower bound the risk of this hypothesis.

\subsection{Learning cluster representation improves risk}
\label{cluster}

We present an example where unsupervised representation learning provably reduces risk. The example uses cluster structure in the unlabeled data to learn a representation where each cluster is mapped to a one-hot code, as shown in Figure \ref{clusteR_Assumption_fig}.
Assuming that the hypothesis learner learns a linear separator and that points within clusters share labels, in the new representation low risk will be achieved. In the original input space, any linear separator will achieve risk greater than some strictly positive threshold if the clusters are not linearly separable. Hence it is possible to prove that representation learning offers a benefit, as formalized in Theorem \ref{cluster_theorem}, which instantiates Theorem \ref{risk_gap_theorem}. This result will be meaningful when the risk gap is positive (see Figure \ref{example_results_fig} for an example).

\begin{theorem}
\label{cluster_theorem}
Let $\hat{R}_a(P_X)$ be the result of the cluster property test described in Algorithm \ref{cluster_Algorithm} run on an unlabeled sample $S_u$. Let $s$ be a side length parameter and $k$ be the number of clusters (see Section \ref{cluster_cond_abcd}). Let $P_{XY}$ satisfy Assumptions \ref{cluster_PXY_Assumption} and \ref{cluster_PXY_Assumption2}, $f_L$ be the feature learner and $h_L$ be the hypothesis learner in Section \ref{cluster_cond_abcd}. Let $\beta$ be a lower bound on the empirical risk of a hypothesis learned by $h_L$ on a training set constructed by adding labels to $S_u$ according to some distribution $P_{XY}$ such that $R_B(P_{XY})\leq\epsilon_B$ and $R_E(P_{XY})\geq\epsilon_E$ (see Supplement \ref{cluster_PX_hL_lemma_proof}). Let $\epsilon_{\min}:=\beta-\sqrt{\frac{8(n+1)\log\frac{2em_u}{n+1}+8\log\frac{12}{\delta}}{m_u}}$ and $\epsilon^Z_{\max}:= \frac{1}{m_u}(s^{-n}\log2+\log\frac{3}{\delta}) + \underset{t \in [0,1]}{\max}(k+1-\frac{\delta}{3}(1-t)^{-m_l})t$.  Suppose $\hat{R}_a(P_X)=0$. Then if hypotheses $h$ and $h^Z \circ f$ are constructed from $S_u$ and a labeled sample $S_l$, with probability at least $1-\delta$, $\Delta R:=R(h)-R(h^Z \circ f) \geq \epsilon_{\min}-\epsilon_{\max}^Z$.
\end{theorem}

\subsection{Conditions $A(P_X), B(P_{XY}), C(f_L), D(h_L)$}
\label{cluster_cond_abcd}

We consider a property which measures the extent to which $P_X$ is concentrated on disjoint clusters and describe an algorithm for testing whether this property approximately holds from a finite unlabeled sample. The quantity $R_A(P_X) \in [0,1]$, defined below, describes the extent to which the property holds, with $R_A(P_X)=0$ indicating that the property perfectly holds.

Given a distribution $P_X$ and a hypercube side length parameter $s$, let $\mathcal{X}_A$ be the set of all sets $X_a$ composed of disjoint regions and for which every point in a region included in $X_a$ is near some point supported by $P_X$ (see Supplement \ref{cluster_PX_theorem_proof} for a formal definition).
For some set $X_a$, let $k=|X_a|$, $L_a(x):=\mathbf{1}(x \not\in \bigcup\limits_{X_i \in X_a}X_i)$ and $\hat{R}_a(P_X)$ be the result of the property test described in Algorithm \ref{cluster_Algorithm}, where if $\hat{R}_a(P_X)=0$ then $L_a(x)=0$ for all $x \in S_u$. Let $R_a(P_X):=\mathbb{E}_{x \sim P_X}[L_a(x)]$ and $R_A(P_X):=\underset{X_a \in \mathcal{X}_A}{\min}R_a(P_X)$.

\begin{lemma}
\label{cluster_PX_theorem}
Let $S_u$ be a sample drawn from $P_X$ and let $\hat{R}_a(P_X)$ be calculated using $S_u$ and the property test described in Algorithm \ref{cluster_Algorithm}. Let $\hat{\epsilon}_A:=0$ and $\epsilon_A:=\frac{1}{m_u}(s^{-n}\log2+\log\frac{3}{\delta})$. With probability at least $1-\frac{\delta}{3}$, $\hat{R}_a(P_X)\leq\hat{\epsilon}_A \implies R_a(P_X)\leq\epsilon_A \wedge R_A(P_X)\leq\epsilon_A$.
\end{lemma}

We adopt a variant of the cluster assumption, which has previously been used in the context of semi-supervised learning \citep{rigollet_generalization_2007,singh_unlabeled_2009}. We assume that nearby points share labels, given that the property test for clusteredness passes (see Section \ref{cluster_cond_abcd}). We set $\epsilon_B:=0$, indicating strict within-cluster label agreement, but envisage relaxing this assumption in future work.

Let $\gamma:=\frac{s}{\sqrt{n}}$ be a cluster separation parameter. Let $d_\gamma$ be a distance function such that $d_\gamma(x_0,x_G)=0$ if there is some path $x_0, \dots , x_G$ such that $\bigwedge\limits_{i=0}^{G-1}(||x_i-x_{i+1}||_2\leq \gamma) \wedge \bigwedge\limits_{i=0}^{G}(p(x_i)>0)$, otherwise $d_\gamma(x_0,x_G)=1$. Let $R_B(P_{XY}):=\mathbb{E}_{\{x,y\},\{x',y'\} \sim P_{XY}}[L_B(\{x,y\},\{x',y'\})]$, where $L_B(\{x,y\},\{x',y'\}):=\mathbf{1}(d_\gamma(x,x')=0)\mathbf{1}(y\neq y')$.

\begin{assumption}
\label{cluster_PXY_Assumption}
Let $\epsilon_B:=0$. Assume $R_A(P_X)\leq\epsilon_A \implies R_B(P_{XY})\leq\epsilon_B$.
\end{assumption}

Having established that the marginal distribution $P_X$ has the property that the data lies in clusters, and making the assumption that points within clusters share labels, we now state the condition that the unsupervised feature learner $f_L$ exploits this cluster structure. We design the learner to produce a feature map $f$ which maps all points within a cluster to the same point in the feature space.

Define $f_L$ as follows, yielding the feature map $f$. Run the property test described in Algorithm \ref{cluster_Algorithm}, which we assume passes and returns the set $X_a$. Set $k=|X_a|$. For $X_i \in X_a$, for all $x \in X_i$ set $f(x)$ to be a $k$-dimensional vector whose $i$th position is 1, and whose other positions are 0. For those points $x \in X$ for which $x \not\in X_i$ for all $X_i \in X_a$, set $f(x)$ to be the zero vector. Let $R_C(f):=\mathbb{E}_{x \sim P_X}[L_C(x,f)]$, where $L_C(x,f):=\underset{\{x': f(x)=f(x')\}}{\max}d_\gamma(x,x')$.

\begin{lemma}
\label{clusteR_FL_theorem}
Let $\epsilon_C:=\epsilon_A$. $R_a(P_X)\leq \epsilon_A \implies R_C(f)\leq\epsilon_C$.
\end{lemma}

Let $h_L$ be a learner which conducts empirical risk minimization over all linear classifiers of type $X \to Y$, using the labeled sample $S_l$. Let $h_L^Z$ be a learner which conducts empirical risk minimization over all linear classifiers of type $Z \to Y$, using the transformed labeled sample $S_l^Z$, yielding the hypothesis $h_L^Z \circ f$. Note that the bound on $R(h^Z \circ f)$ shown is independent of $P_{XY}$ given $R_B(P_{XY})\leq\epsilon_B$ and $R_C(f)\leq\epsilon_C$.

\begin{lemma}
\label{cluster_hL_theorem}
Let $\epsilon_B:=0$ and $\epsilon_{\max}^Z:=\epsilon_C + \underset{t \in [0,1]}{\max}(k+1-\frac{\delta}{3}(1-t)^{-m_l})t$.
With probability at least $1-\frac{\delta}{3}$, $R_B(P_{XY})\leq\epsilon_B \wedge R_C(f)\leq\epsilon_C \implies R(h^Z \circ f)\leq\epsilon_{\max}^Z$.
\end{lemma}

\subsection{Conditions $E(P_{XY}), F(P_X,h_L)$}
\label{cluster_cond_ef}

We provide a lower bound on the risk of learning in the original input space rather than the learned representation. For this lower bound to be meaningful, we require that the joint distribution is `difficult' to learn in some way. In particular, we assume that there is no label which is correct for almost all inputs.

Without loss of generality assume $P_{XY}$ has the property $p(y=0)\geq p(y=1)$. Let $R_E(P_{XY})=\mathbb{E}_{\{x,y\} \sim P_{XY}}[L_E(\{x,y\})]$, where $L_E(\{x,y\})=y$.

\begin{assumption}
\label{cluster_PXY_Assumption2}
Let $\epsilon_E\in(0,\frac{1}{2}]$ be specified by a domain expert. Assume that $R_E(P_{XY})\geq\epsilon_E$.
\end{assumption}

Recall that $h_L$ is a learner conducting empirical risk minimization on the labeled sample $S_l$ over the set $H$ of linear classifiers of type $X \to Y$, yielding the hypothesis $h$. Assume that $h_L$ is guaranteed to select the hypothesis in $H$ with the minimum empirical risk. We now provide a high probability lower bound on $R(h)$, which is also dependent on $P_X$, and will be meaningful when it is greater than zero. Note that this bound is independent of the number of labeled examples $m_l$.

\begin{lemma}
\label{cluster_PX_hL_lemma}
Let $\epsilon_B:=0$, $\epsilon_E\in(0,\frac{1}{2}]$ be specified by a domain expert, $S_u$ be a sample from $P_X$ and $\beta$ be a lower bound on the empirical risk of a hypothesis learned by $h_L$ on a training set constructed by adding labels to $S_u$ according to some distribution $P_{XY}$ such that $R_B(P_{XY})\leq\epsilon_B$ and $R_E(P_{XY})\geq\epsilon_E$ (see Supplement \ref{cluster_PX_hL_lemma_proof}). Let $\epsilon_{\min}:=\beta-\sqrt{\frac{8(n+1)\log\frac{2em_u}{n+1}+8\log\frac{12}{\delta}}{m_u}}$. With probability at least $1-\frac{\delta}{3}$, $\int_yP_{XY}dy=P_X \wedge R_E(P_{XY})\geq\epsilon_E \wedge R_B(P_{XY})\leq\epsilon_B \implies R(h)\geq \epsilon_{\min}$.
\end{lemma}

We have now stated all lemmas used by the proof of Theorem \ref{cluster_theorem}, which is the main result for the cluster representation example.

\section{Conclusion}

We have developed a theory that explains when unsupervised pre-training works --- a set of sufficient conditions which with high probability imply that a change of representation will improve task performance. These conditions depend jointly upon structural properties of the distribution, the feature learner, the subsequent supervised learner and the loss function used.
We instantiated the conditions on an example which exploits the cluster structure, with a second example on manifold structure in the Supplement.
Our approach of analysing a processing pipeline seems novel, where it seems necessary to consider risk gaps instead of sample complexity.

The modular nature of our argument is a central feature; it breaks an obscure black-box into understandable components.  And furthermore, the instantiation of each component, for a particular problem, reduces to relatively standard application of existing techniques.

\small
\bibliographystyle{alpha}
\bibliography{feature-learning}
\normalsize

\newpage
\appendix

\section*{Supplementary material:\\
Learning features to improve task performance}

\section{Motivation for and technical discussion of proposed sufficient conditions}
\label{motivation_technical_discussion}

We sketch the motivation for the sufficient conditions proposed for unsupervised representation learning, as well as technical aspects of these conditions which were omitted from the main text due to space requirements.

\subsection{Motivation for sufficient conditions}
\label{motivation}

The process for evaluating the effect of unsupervised representation learning is described in Figure \ref{flow_chart}. By identifying source nodes which represent independent elements of the prediction context, we associate a condition with each such element, as stated in Table \ref{conditions_table}. Thus, although we do not formally demonstrate that the conditions are \textit{necessary} for unsupervised representation learning, it appears unlikely that the conditions can be reduced further.

Furthermore, the modular structure of the conditions means that the task of determining the effect of representation learning is `factorized' over the different elements of the prediction context, making analysis more tractable. Given an unlabeled data sample, a feature learner, a supervised learner, and suitable assumptions about shared structure between the marginal distribution $P_X$ and the joint distribution $P_{XY}$, it is possible to assess the effect of representation learning with high probability. Note that while $P_X$ is not independent of $P_{XY}$, we treat $P_X$ as an independent element because $P_{XY}$ is unknown, and treat the assumption of shared structure in $P_{XY}$ as a separate independent element.

The conditions imply that the value of a particular representation will depend upon on the hypothesis learner. A representation which improves the performance of a particular hypothesis learner --- for example, empirical risk minimization over the class of linear separators --- may hinder the performance of another learner which can effectively learn in the original input space. Furthermore, abstracting away from a particular hypothesis learner to consider only the Bayes optimal classifier demonstrates that representation learning never helps in this setting \citep{van_rooyen_theory_2015}. Consequently, it appears necessary to include dependence on a subsequent hypothesis learner when defining `good' features.

It is clear that the values of $R(h^Z \circ f)$ and $R(h)$ also depend on the loss function $L$, and hence so do the conditions $D$ and $F$. For brevity we suppress this in the notation used in Table \ref{conditions_table} and Figure \ref{flow_chart}. We also note that there is typically a connection between the loss function and the hypothesis learner. The learner will typically optimize parameters according to the loss function, or some variant of the loss function that makes optimization easier and/or introduces regularization.

We propose an upper bound on $R(h^Z \circ f)$ that is independent of $P_{XY}$ given that $R_B(P_{XY})\leq\epsilon_B$ and $R_C(f)\leq\epsilon_C$ (see Table \ref{conditions_table}). This means that the bound applies to a wide class of distributions and therefore has the advantage of being quite general. It is also somewhat loose, as it does not consider $P_X$ itself, and a tighter bound which includes dependence on $P_X$ may be possible. On the other hand, to achieve a meaningful lower bound on $R(h)$ we require dependence on $P_X$ since it is likely that for some choices of $P_X$, $R(h)$ is zero or small. This is the reason that condition $F$, which is associated with $h_L$ as shown in Figure \ref{flow_chart}, also depends on $P_X$.

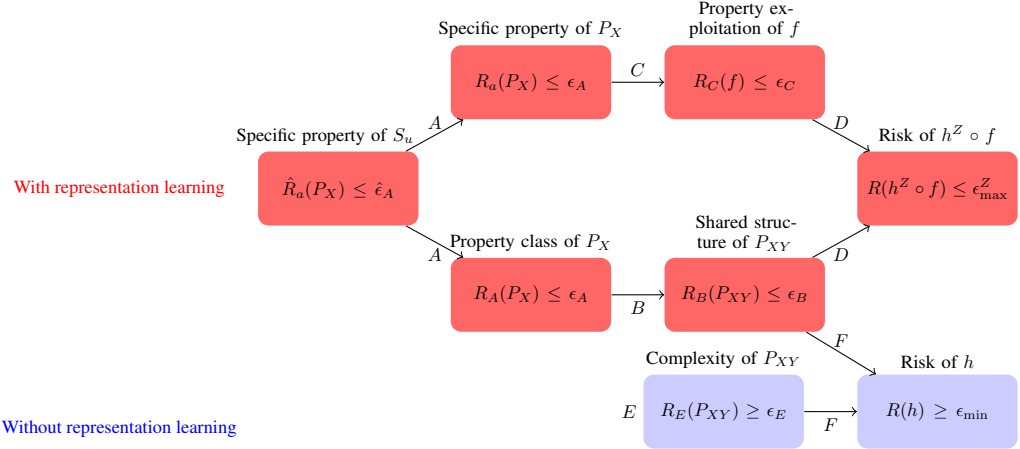
\begin{figure}
\centering

\begin{tikzpicture}[scale=0.7, transform shape]
\tikzstyle{cloud} = [draw, ellipse, node distance=2.5cm,
minimum width=5cm]
\tikzstyle{a} = [rectangle, draw=none, fill=blue!20,
text width=8em, text centered, rounded corners, minimum height=4em]
\tikzstyle{b} = [rectangle, draw=none, fill=red!60,
text width=8em, text centered, rounded corners, minimum height=4em]

\tikzstyle{line} = [draw, -latex']
\node [b] (R_a_emp) {$\hat{R}_a(P_X)\leq\hat{\epsilon}_A$};
\node[b,above right = 0.6cm and 0.6cm of R_a_emp] (R_a) {$R_a(P_X)\leq\epsilon_A$};
\node[b,below right = 0.6cm and 0.6cm of R_a_emp] (R_A) {$R_A(P_X)\leq\epsilon_A$};
\node[b,right = 1cm of R_A] (R_B) {$R_B(P_{XY})\leq\epsilon_B$};
\node[b,right = 1cm of R_a] (R_C) {$R_C(f)\leq\epsilon_C$};
\node[b,above right = 0.6cm and 0.6cm of R_B] (R_h_z) {$R(h^Z \circ f)\leq\epsilon^Z_{\max}$};
\node[a,below right = 0.8cm and 0.6cm of R_B] (R_h) {$R(h)\geq\epsilon_{\min}$};
\node[a,left = 1cm of R_h] (R_E) {$R_E(P_{XY})\geq\epsilon_E$};
\node[left = 0cm of R_E] {$E$};
\node[above = 0cm of R_a_emp, text width=11em, align=left] {Specific property of $S_u$};
\node[above = 0cm of R_a, text width=10em, align=center] {Specific property of $P_X$};
\node[above = 0cm of R_A, text width=10em, align=center] {Property class of $P_X$};
\node[above = 0cm of R_C, text width=10em, align=center] {Property exploitation of $f$};
\node[above = 0cm of R_B, text width=10em, align=center] {Shared structure of $P_{XY}$};
\node[above = 0cm of R_h_z, text width=10em, align=center] {Risk of $h^Z \circ f$};
\node[above = 0cm of R_h, text width=10em, align=center] {Risk of $h$};
\node[above = 0cm of R_E, text width=10em, align=center] {Complexity of $P_{XY}$};
\draw [->](R_a_emp) -- node[above] {$A$} ++(R_a);
\draw [->](R_a_emp) -- node[below] {$A$} ++(R_A);
\draw [->](R_A) -- node[below] {$B$} ++(R_B);
\draw [->](R_a) -- node[above] {$C$} ++(R_C);
\draw [->](R_B) -- node[below] {$D$} ++(R_h_z);
\draw [->](R_C) -- node[above] {$D$} ++(R_h_z);
\draw [->](R_B) -- node[above] {$F$} ++(R_h);
\draw [->](R_E) -- node[below] {$F$} ++(R_h);
\node [left =0.5cm of R_a_emp,red](feature_learning) {With representation learning};
\node [below =4cm of feature_learning,blue](no_feature_learning) {Without representation learning};
\end{tikzpicture}

\caption{Visualization of the conditions for unsupervised representation learning to reduce risk, as described in Table \ref{conditions_table}. The statements in red boxes (above) provide an upper bound on risk using representation learning, while the statements in blue boxes (below) are additionally required to provide a lower bound on risk without using representation learning. A description of what each statement measures is provided above its box. Arrows indicate conditional implications and are labeled with the condition from Table \ref{conditions_table} they depend upon.} \label{dependency_map}
\vspace{-15pt}
\end{figure}

\subsection{Implication structure of sufficient conditions}
\label{technical_discussion}

The conditions allow a property test to assess the effect of representation learning through decomposition into a series of implications, as shown in Figure \ref{dependency_map}. The property test conducted on $S_u$ implies (condition $A$) both that $P_X$ has some \textit{specific property}, which in our examples is concentration on a particular region, and that $P_X$ has some property within a \textit{property class}, which in our examples is concentration on some region of a particular type. The \textit{specific property} of $P_X$ allows an appropriately constructed feature map $f$ to preserve the useful information in $P_X$ while moving to an input space that is simpler for a particular subsequent hypothesis learner to learn from (condition $C$). The \textit{property class} of $P_X$  is a more natural piece of information to provide to a domain expert with a view to eliciting shared structure in $P_{XY}$ (condition $B$).

The upper bound on $R(h^Z \circ f)$ depends both on the fact that $f$ simplifies the structure of the inputs to be compatible with hypothesis learner $h_L^Z$, and that the input structure is shared with the joint distribution $P_{XY}$ (condition $D$). The lower bound on $R(h)$ involves adding labels to the unlabeled data, subject to some constraints, and using the results of empirical risk minimization on `best case scenario' labels to provide a lower bound (condition $F$). Such constraints can be constructed using the existing assumption about $P_{XY}$ required for the upper bound on $R(h^Z \circ f)$, plus an additional assumption about the complexity of $P_{XY}$ (condition $E$).

\subsection{Interpretation of high probability statements}
\label{high_probability_interpretation}

The statements made with probability at least $1-\delta$ should be interpreted as follows. For any distribution $P_{XY}$, if an unlabeled sample $S_u$ and a labeled sample $S_l$ are drawn, the following statements (where applicable) will all be true with probability at least $1-\delta$.
\begin{enumerate}
\item For all regions in some set of regions specified independently of $S_u$, the total probability mass lying outside the region, $R_a(P_X)$,  is not too much greater than the empirical estimate $\hat{R}_a(P_X)$ calculated from $S_u$. (both examples, see Lemmas \ref{manifold_PX_theorem} and \ref{cluster_PX_theorem})
\item For a set of regions specified independently of $S_l$, there exists some subset of these regions containing at least some fixed proportion of probability mass, for which $S_l$ includes at least one point within each region in the subset. (both examples, see Lemmas \ref{manifold_hL_lemma} and \ref{cluster_hL_theorem})
\item For all regions in some set of regions specified independently of $S_u$, the total probability mass lying inside the region, $p_c$,  is not too much less than the empirical estimate $\hat{p}_c$ calculated from $S_u$. (manifold example only, see Lemma \ref{manifold_fL_lemma})
\item For a fixed hypothesis class $H$, the risk $R(h)$ of a hypothesis $h \in H$ learned using some labeling of $S_u$ is not too much less than some empirical quantity $\beta$ calculated from $S_u$. (cluster example only, see Lemma \ref{cluster_PX_hL_lemma})
\end{enumerate}

\subsection{Concept of transparent polymorphism}
\label{transparent_polymorphism}

Our analysis allows the comparison of a pair of (potentially unrelated) supervised learners, but focuses on `transparently polymorphic' learners to isolate the effect of the representation learning step. These are families of learning algorithms such that given a change of type signature the algorithm can straightforwardly be extended to the new type. The examples given in this text --- empirical risk minimization over the class of linear classifiers, and 1-nearest neighbor using Euclidean distance --- are examples of such learners. A precise definition of this class of learners is not attempted here and does not appear to have been well-studied elsewhere. A quantitative measure of the transparent polymorphism of a particular learner is of interest in its own right given the ubiquity of such learners in practical applications. It is easy to construct a learner which is \textit{not} transparently polymorphic: for example, a learner which conducts logistic regression if it receives inputs with ten dimensions or fewer, and which uses a support vector machine if it receives inputs with more than ten dimensions.  The same issue arises in relation to the number of labeled examples available to the learner. A `transparently polymorphic' learner should behave predictably when the the number of labeled examples available to it is varied.

\section{Generalization error bounds in semi-supervised learning}
\label{ssl_bound}

A formalization of semi-supervised learning is provided in \cite{balcan_discriminative_2010}, including improved generalization error bounds compared to supervised learning. We describe this approach to build on the comparison between unsupervised representation learning and semi-supervised learning presented in Section \ref{semisupervised_relationship}. For a hypothesis class $H$ and input space $X$, define a compatibility function $\chi: H \times X \to [0,1]$. The incompatibility of a hypothesis $h$ with respect to an unlabeled distribution $P_X$ is defined as $R_U(h)=1-\mathbb{E}_{x \sim P_X}[\chi(h,x)]$, while the incompatibility of a hypothesis with respect to an unlabeled sample $S_u$ of $m_u$ points is $\hat{R}_{U}(h)=1-\frac{1}{m_u}\sum\limits_{x \in S_u}\chi(h,x)$.

The benefit of semi-supervised learning rests on the idea that for the target function $h^*$, $R_{U}(h^*)=0$ in the simplest case, which we consider here, or in other cases $R_{U}(h^*)$ is small. Moreover, for only a small subset of hypotheses $H_{P_X,\chi}(\epsilon) \subset H$, $R_{U}(h)\leq\epsilon$ for all $h \in H_{P_X,\chi}(\epsilon)$, so that the size of the hypothesis class is reduced. In this case we assume that the hypothesis class $H$ is finite. For a hypothesis for which $\hat{R}_{U}(h)=0$ and $\hat{R}(h)=0$, and a sufficiently large number of unlabeled examples $m_u$, the number of labeled examples $m_l$ required to show that that with probability at least $1-\delta$, $R(h)\leq\epsilon$ is:

\begin{equation}
m_l=\frac{1}{\epsilon}(\log|H_{P_X,\chi}(\epsilon)|+\log\frac{2}{\delta})
\end{equation}

Assuming that the compatibility function $\chi$ and marginal distribution $P_X$ allow a significant subset of the hypothesis space $H$ to be pruned, the sample complexity is reduced in comparison to the typical supervised learning requirement (see Lemma \ref{finite_lemma}).

The result can also be generalized to infinite hypothesis classes, although in this case it is not possible to directly obtain improved sample complexity using VC-dimension as a measure of hypothesis class size. Instead the authors propose the use of a distribution dependent measure,  $H_{P_X,\chi}(\epsilon)[m,P_X]$, which measures the expected number of splits of $m$ points drawn from $P_X$, using hypotheses $h \in H$ such that $R_{U}(h)\leq\epsilon$.

\section{Algorithm for selecting a feature learner}
\label{algorithm_uor_selecting_A_feature_learner}

Algorithm \ref{feature_learning_Algorithm} uses risk upper bounds to select a feature learner. It is presented as a high-level conceptual algorithm rather than as a tool for immediate practical use. The algorithm accepts a set of possible feature learners $\mathcal{F}_L$, an unlabeled sample $S_u$, a hypothesis learner $h_L$, the number of labeled samples $m_l$, the level of certainty parameter $\delta$ and a dictionary of bounds. Each bound in the dictionary contains a property test for $P_X$ based on an unlabeled sample, a feature learner $f_L$ and a hypothesis learner $h_L$ to which the bound applies, and appropriate assumptions on $P_{XY}$ based on the results of the property test. Two examples of such bounds are provided in Section \ref{examples} of this paper.

For each bound, the algorithm runs the associated property test and using this computes the bound result. If the bound is tighter than any previous bound, then the current feature learner is selected as the best to date. The list of bounds should include a bound for the feature learner which simply returns the identity function. This approach is similar to structural risk minimization, which selects the hypothesis space minimizing a probabilistic upper bound on risk \citep{vapnik_nature_2013}.

\begin{algorithm}
\SetKwProg{myproc}{Function}{}{}
\myproc{select\_feature\_learner($\mathcal{F}_L$,$S_u$,$h_L$,$m_l$,$\delta$,bound\_dictionary)}{

	\textit{best\_bound\_result}$:=$Inf,\textit{ best\_$f_L$}$:=$null \\
	\For{each $f_L$ in $\mathcal{F}_L$}{
		\For{each bound in bound\_dictionary}{
			\If{bound[$f_L$]$=$$f_L$ and bound[$h_L$]$=$$h_L$}{
				\textit{property\_test\_result}$:=$\textit{property\_test}(\textit{bound[test]},$S_u$) \\
				\textit{bound\_result}$:=$\textit{compute\_bound}(\textit{bound},\textit{property\_test\_result},$m_l$,$\delta$) \\
				\If{bound\_result$<$best\_bound\_result}{
					\textit{best\_bound\_result$:=$bound\_result} \\
					\textit{best\_$f_L:=f_L$}
				}
			}
		}
	}
	\Return [\textit{best\_$f_L$}, \textit{best\_bound\_result}]
 }
\vspace{10pt}
\caption{Selecting a feature learner using risk upper bounds}
\label{feature_learning_Algorithm}
\end{algorithm}

\section{Main theorem proofs and lemmas used in examples}

We present the proofs of the main theorems. We also introduce a technical lemma and present standard generalization error bounds which are used in the proofs for the theorems associated with the examples.

\subsection{Proof of Theorem \ref{sample_complexity_theorem}}

\begin{proof}
By the definitions in Table \ref{conditions_table}, $\hat{R}_a(P_X)\leq\hat{\epsilon}_A \wedge A(P_X) \implies R_a(P_X)\leq\epsilon_A \wedge R_A(P_X)\leq\epsilon_A$. $ R_A(P_X)\leq\epsilon_A \wedge B(P_{XY}) \implies R_B(P_{XY})\leq\epsilon_B$. Also, $R_a(P_X)\leq\epsilon_A \wedge C(f_L) \implies R_C(f)\leq\epsilon_C$. Furthermore, $D(h_L) \wedge R_B(P_{XY})\leq\epsilon_B \wedge  R_C(f)\leq\epsilon_C \implies  R(h^Z \circ f)\leq \epsilon_{\max}^Z$. Therefore $\hat{R}_a(P_X)\leq\hat{\epsilon}_A \wedge A(P_X) \wedge B(P_{XY}) \wedge C(f_L) \wedge D(h_L) \implies  R(h^Z \circ f)\leq \epsilon_{\max}^Z$.
If $\hat{R}_a(P_X)\leq\hat{\epsilon}_A$ and with probability at least $1-\delta$, $A, B, C$ and $D$ all hold, then with at least the same probability the right hand side of the statement holds.
\end{proof}
\vspace{-20pt}
\subsection{Proof of Theorem \ref{risk_gap_theorem}}

\begin{proof}
By Theorem \ref{sample_complexity_theorem}, $\hat{R}_a(P_X)\leq\hat{\epsilon}_A \wedge A(P_X) \wedge B(P_{XY}) \wedge C(f_L) \wedge D(h_L) \implies  R(h^Z \circ f)\leq \epsilon_{\max}^Z$. Similarly, $\hat{R}_a(P_X)\leq\hat{\epsilon}_A \wedge A(P_X) \wedge B(P_{XY}) \wedge E(P_{XY}) \wedge F(P_X,h_L) \implies  R(h)\geq \epsilon_{\min}$.
If $\hat{R}_a(P_X)\leq\hat{\epsilon}_A$ and with probability at least $1-\delta$, $A, B, C, D, E$ and $F$ all hold, then with at least the same probability the right hand sides of both statements hold.
\end{proof}
\vspace{-20pt}

\subsection{Technical lemma on sample coverage of finite number of bins}
\label{alpha_discussion}

\begin{figure}
\includegraphics[scale=0.23]{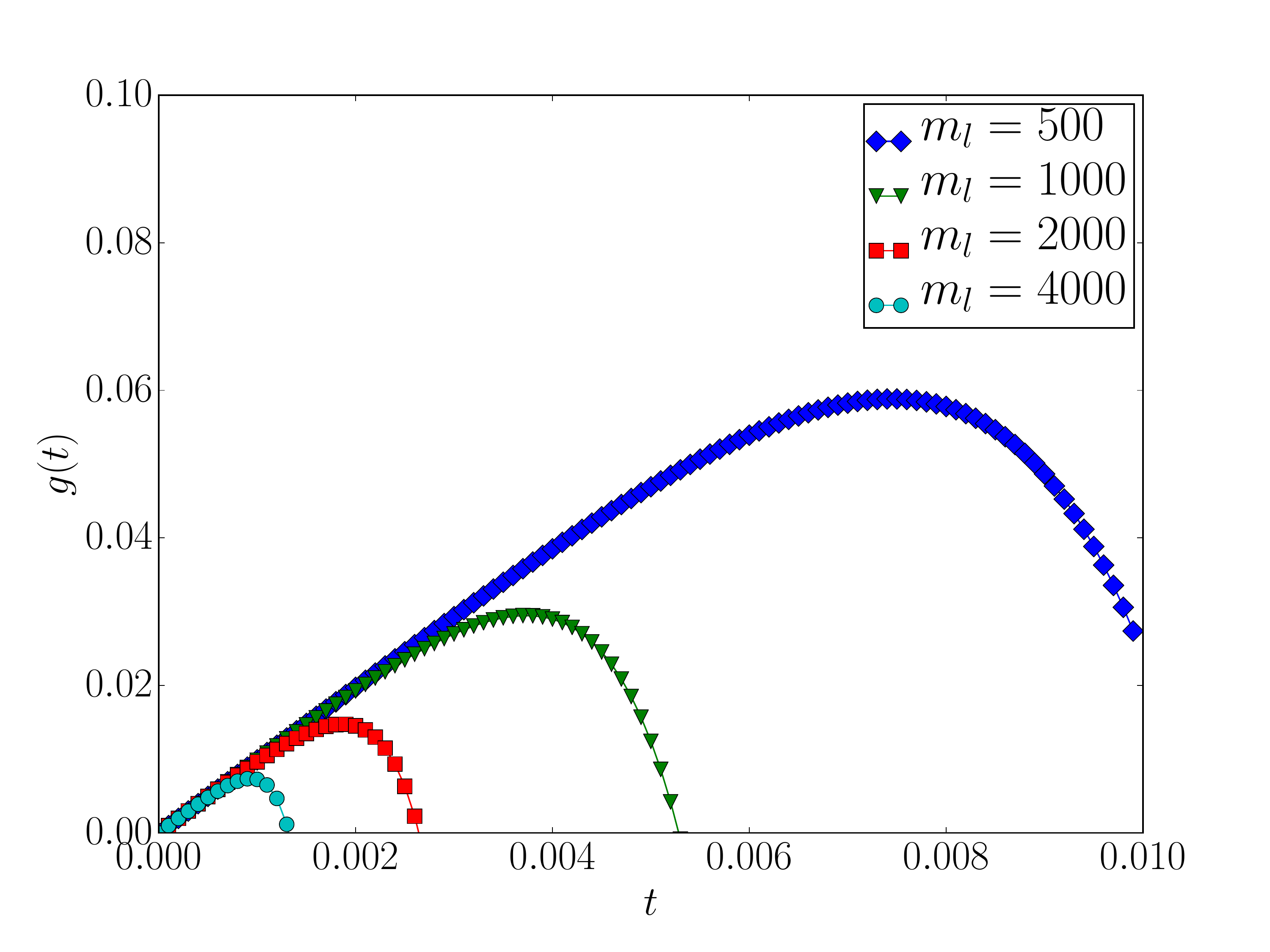}
\includegraphics[scale=0.23]{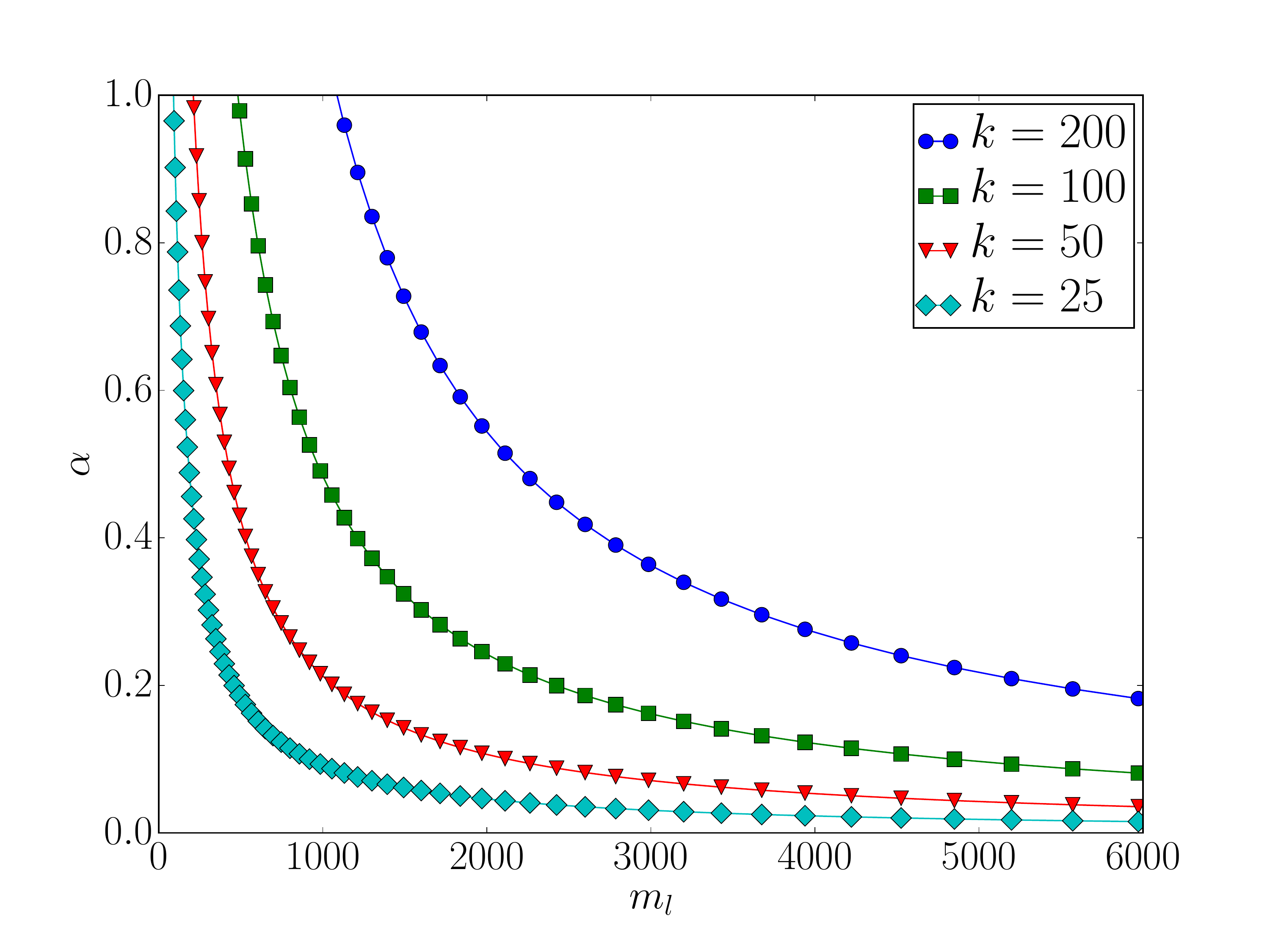}
\vspace{-15pt}
\caption{The function $g(t)=(k-\delta(1-t)^{-m_l})t$ to be maximized is a concave function (left), where the parameters used are $k=10$ and $\delta=0.05$. The quantity $\alpha:=\underset{t \in [0,1]}{\max}g(t)$ (right) is shown for various values of $m_l$ and $k$, where again $\delta=0.05$.} \label{alpha_fig}
\end{figure}

We introduce and analyze a technical lemma which will be of use in establishing condition $D(h_L)$ in our examples. In a setting where a labeled sample is drawn from a finite number of bins, Lemma \ref{alpha_lemma} provides a probabilistic upper bound on the total probability mass of the bins containing no labeled points. Figure \ref{alpha_fig} demonstrates that the bound is tightest when the labeled sample size is large relative to the number of bins.

\begin{lemma}
\label{alpha_lemma}
Assume that a labeled sample $S_l$ of size $m_l$ is drawn from $k$ bins. With probability at least $1-\delta$, the total probability mass of the bins containing no labeled points is no greater than $\alpha:=\underset{t \in [0,1]}{\max}g(t)=(k-\delta(1-t)^{-m_l})t$. Furthermore, $g(t)$ is concave.
\end{lemma}

\begin{proof}
Within each of bin $i$ lies some probability mass $p_i$. Now consider the set $Q$ containing the $q$ bins with highest probability mass. Let $t=\underset{Q}{\min}p_i$. With probability at least $1-q(1-t)^{m_l}$, all of the $q$ intervals contain at least one point in $S_l$. This can be shown by observing that for a single interval it will be empty with probability at most $(1-t)^{m_l}$ and then applying the union bound over the $q$ intervals.  Observe that the remaining intervals have mass at most $(k-q)t$ and eliminate $q$ by setting $\delta=q(1-t)^{m_l}$, yielding an upper bound on the mass of the remaining intervals of $(k-\delta(1-t)^{-m_l})t=:g(t)$ if $t$ is known. If $t$ is unknown, the upper bound is $\alpha:=\underset{t \in [0,1]}{\max}g(t)$. Note that $q\leq \frac{1}{t}$ so that the equation $\delta\leq \frac{1}{t}(1-t)^{m_l}$ provides an upper bound on $t$. $g(t)$ is concave since $\frac{d^2g}{dt^2}=-\delta m_l(1-t)^{-(m_l+1)}(1+\frac{t(m_l+1)}{1-t})$ is non-positive for $t \in [0,1]$.
\end{proof}

\subsection{Standard generalization error bounds}
We present standard generalization error bounds which we make use of in subsequent lemmas. Lemma \ref{finite_lemma} is used in the case of finite hypothesis classes, while Lemma \ref{infinite_lemma} is used in the case of infinite hypothesis classes. The proofs for both lemmas appear in \cite{mohri_foundations_2012}. In the case of Lemma \ref{infinite_lemma}, the bound is presented as an upper bound on $R(h)$ with the square root term on the right hand side added to rather than subtracted from $\hat{R}(h)$. However, a symmetric lower bound can be straightforwardly derived using the same proof.

\begin{lemma}
\label{finite_lemma}
For a hypothesis $h$ drawn from a finite hypothesis class $H$ learned from $m$ examples for which $\hat{R}(h)=0$, with probability at least $1-\delta$, $R(h) \leq \frac{1}{m}(\log|H|+\log\frac{1}{\delta})$.
\end{lemma}

\begin{lemma}
\label{infinite_lemma}
Let $d$ be the VC dimension of hypothesis class $H$. For all hypotheses $h$ drawn from $H$ learned from $m$ examples, with probability at least $1-\delta$, $ R(h) \geq \hat{R}(h)-\sqrt{\frac{8d\log\frac{2em}{d}+8\log\frac{4}{\delta}}{m}}$.
\end{lemma}

\section{Supplementary material for the cluster example}

The parameters used for the cluster example in Section \ref{cluster} are summarized in Table \ref{cluster_table}. Proofs for Theorem \ref{cluster_theorem} and the lemmas used are shown subsequently. For a given set of parameters, we plot the risk gap induced by the learned features in Figure \ref{example_results_fig} (right), which depends both on the number of unlabeled points $m_u$ and the number of labeled points $m_l$. In particular, the learned features are provably useful only when the risk gap is positive. This is not always the case, which is not surprising since we expect that a given representation learning technique will only be useful in certain prediction contexts.

\begin{table}[h!]
\small

\begin{center}
\begin{tabular}{p{1.3cm}|p{2cm}|p{5.8cm}|p{3.3cm}}
\bf Condition & \bf Description &\bf Risk function $R$ & \bf Threshold $\epsilon$ \\ \hline
\multicolumn{4}{l}{\it Upper bound on risk using representation learning} \\ \hline
$A(P_X)$ & If property test passes, marginal distribution is concentrated on disjoint clusters & $R_a(P_X)=\mathbb{E}_{x \sim P_X}[L_a(x)]$, where $L_a(x)=\mathbf{1}(x \not\in \bigcup\limits_{X_i \in X_a}X_i)$.  $\hat{R}_a(P_X)$ is the result of Algorithm \ref{cluster_Algorithm} and $R_A(P_X)=\underset{X_a \in \mathcal{X}_A}{\min}R_a(P_X)$. See Appendix \ref{cluster_PX_theorem_proof} for the definition of $\mathcal{X}_A$.  & $\hat{\epsilon}_A=0$ \newline $\epsilon_A=\frac{1}{m_u}(s^{-n}\log2+\log\frac{3}{\delta})$ \\ \hline
$B(P_{XY})$ & Given cluster structure exists, nearby points share labels & $R_B(P_{XY})=\mathbb{E}_{\{x,y\},\{x',y'\} \sim P_{XY}}[L_B(\{x,y\},\{x',y'\})]$, where $L_B(\{x,y\},\{x',y'\})=\mathbf{1}(d_\gamma(x,x')=0)\mathbf{1}(y\neq y')$ and $d_\gamma$ is defined in Section \ref{cluster_cond_abcd}& $\epsilon_B=0$  \\ \hline
$C(f_L)$ & Feature learner uses cluster structure to learn one-hot code for each cluster  &$R_C(f)=\mathbb{E}_{x \sim P_X}[L_C(x,f)]$, where $L_C(x,f)=\underset{\{x': f(x)=f(x')\}}{\max}d_\gamma(x,x')$ & $\epsilon_C=\epsilon_A$  \\ \hline
$\rule{0pt}{3ex}$$D(h_L)$ &Clusters are approximately linearly separable in new representation& $R(h^Z \circ f)=\mathbb{E}_{\{x,y\} \sim P_{XY}}[L(h^Z(f(x)),y)]$ where $L(y,y')=\mathbf{1}(y\neq y')$ & $\epsilon^Z_{\max}= \epsilon_C + \underset{t \in [0,1]}{\max}(k+1-\frac{\delta}{3}(1-t)^{-m_l})t$ \\ \hline
\multicolumn{4}{l}{\it Lower bound on risk without using representation learning} \\ \hline
$\rule{0pt}{2ex}$
$E(P_{XY})$ & There is no single label which is correct for most points &$R_E(P_{XY})=\mathbb{E}_{\{x,y\} \sim P_{XY}}[L_E(\{x,y\})]$, where $L_E(\{x,y\})=\mathbf{1}(y=1)$ and without loss of generality we assume $P(y=0)\geq P(y=1)$ & $\epsilon_E\in(0,\frac{1}{2}]$ is specified by a domain expert \\ \hline
$F(P_X,h_L)$ & The clusters are not approximately linearly separable in the original representation &$R(h)$, using the definition of $R$ provided above in condition $D(h_L)$ & $\epsilon_{\min}=\beta-\sqrt{\frac{8(n+1)\log\frac{2em_u}{n+1}+8\log\frac{12}{\delta}}{m_u}}$, where $\beta$ is defined in Appendix \ref{cluster_PX_hL_lemma_proof}. \\
\end{tabular}
\end{center}
\caption{Parameters for cluster example} \label{cluster_table}
\end{table}

\subsection{Kernel interpretation}
\label{cluster_kernel_interpretation}

It is possible to obtain a dual form of linear classifiers working with a kernel function $k$ in the original input space, which is equivalent to taking a dot product in the feature space associated with the feature map $f$. This may be written as $k(x,x')=\langle f(x),f(x') \rangle$. By the definition of $f(x)$ stated in Section \ref{cluster_cond_abcd}, and setting $c_x$ and $c_{x'}$ to be the hypercubes in which $x$ and $x'$ lie respectively, we obtain:

$
k(x,x')=
\begin{cases}
1,& \text{if } \exists x_0,\dots,x_G \in S_u \text{ s.t. }x_0 \in c_x \wedge x_G \in c_{x'} \wedge \bigwedge\limits_{i=0}^{G-1}||x_{i}-x_{i+1}||_2\leq s\sqrt{n}\\
0, & \text{otherwise.}
\end{cases}
$

Thus the proposed representation can be viewed as a form of unsupervised kernel learning. However, in the general case of hypothesis learner $h_L$, no kernel interpretation is possible because hypotheses learned will not in general have a dual form like linear classifiers.

\subsection{Proof of Theorem \ref{cluster_theorem}}
\label{cluster_theorem_proof}
\begin{proof}
With probability at least $1-\delta$, by the union bound, we have $A(P_X)$ by Lemma \ref{cluster_PX_theorem}, $B(P_{XY})$ by Assumption \ref{cluster_PXY_Assumption}, $C(f_L)$ by Lemma \ref{clusteR_FL_theorem}, $D(h_L)$ by Lemma \ref{cluster_hL_theorem}, $E(P_{XY})$ by Assumption \ref{cluster_PXY_Assumption2} and $F(P_X,h_L)$ by Lemma \ref{cluster_PX_hL_lemma}. Applying Theorem \ref{risk_gap_theorem}, we have $\Delta R \geq \epsilon_{\min}-\epsilon_{\max}^Z$.
\end{proof}

\subsection{Proof of Lemma \ref{cluster_PX_theorem}}
\label{cluster_PX_theorem_proof}
First, given a distribution $P_X$ and a side length parameter $s$, let $\gamma:=\frac{s}{\sqrt{n}}$ and let $\mathcal{X}_A$ be the set of all sets $X_a$ with the following properties:

\begin{enumerate}

\item $X_a=\bigcup\limits_{i=1}^k\{X_i\}$ for some finite $k\geq2$, where $X_i \subset X$ for all $i\leq k$

\item $||x-x'||_2 > \gamma$ for all $x \in X_i$, $x' \in X_j$, $i \neq j$

\item For all $x \in X_i$ there exists some point $x'$ supported by $P_X$ such that $||x-x'||_2\leq\gamma$.

\end{enumerate}
The algorithm for testing for and learning cluster structure from unlabeled data is presented in Algorithm \ref{cluster_Algorithm}. The algorithm returns failure if it cannot find at least two disjoint clusters that all of $S_u$ lie within. In future this restriction could be relaxed to find clusters that \textit{most} of the data lies within, allowing us to consider the case where $\hat{\epsilon}_A>0$.
\begin{algorithm}

 \caption{Testing for and learning cluster structure}
 \label{cluster_Algorithm}
 \KwData{unlabeled sample $S_u$, input space $X$, side length parameter $s$}
 \KwResult{Clusters ($X_a$) and proportion of $S_u$ not concentrated on clusters ($\hat{R}_a$), or failure}
 \textit{Divide $X$ into hypercubes of side $s$} \\
 \textit{Set $\gamma:=\frac{s}{\sqrt{n}}$} \\
 \textit{Build graph from $S_u$, placing an edge between two points $x$ and $x'$ if $||x-x'||_2\leq\gamma$} \\
 \textit{Extract components of graph, where $C_i$ is the set of points in the $i$th component} \\
 \textit{Set $X_a:=\{\}$} \\
 \For {each component $C_i$}{
   Set $X_i:=\{\}$ \\
    \For {each point $x \in C_i$}{
       \textit{Find the hypercube $c_x$ such that $x \in c_x$} \\
       \textit{Set $X_i := X_i \cup c_x$}
     }
      \textit{Set $X_a:=X_a \cup \{X_i\}$}
   }
  \eIf{$|X_a|\geq 2$}{
$\hat{R}_a:=0$ \\
\Return $\{X_a,\hat{R}_a\}$\;
   }{
\Return \textit{failure}\
  }
\end{algorithm}

We now state the proof of Lemma \ref{cluster_PX_theorem}.

\begin{proof}
If $\hat{R}_a(P_X)=0$, then Algorithm \ref{cluster_Algorithm} succeeds and also returns some $X_a \in \mathcal{X}_A$.
For each set $X_a \in \mathcal{X}_A$ define a hypothesis $h_a$, where $h_a(x)=1$ if $x \in \bigcup\limits_{X_i \in X_a}X_i$, and $h_a(x)=0$ otherwise. Let $H_A$ be the set of such hypotheses, where $|H_A|\leq2^{s^{-n}}$. Let $R(h_a):=\mathbb{E}_{x \sim P_X}[L(h_a(x),1)]$ and $\hat{R}(h_a):=\frac{1}{m_u}\sum\limits_{x \in S_u}L(h_a(x),1)$, recalling that $L$ is 0/1 loss.

Because $\hat{R}(h_a)=\frac{1}{m_u}\sum\limits_{x \in S_u}L_a(x)$ by definition, and $\frac{1}{m_u}\sum\limits_{x \in S_u}L_a(x)=0$ if $\hat{R}_a(P_X)=0$ by the construction of Algorithm \ref{cluster_Algorithm}, we have $\hat{R}(h_a)=0$. Hence we may apply Lemma \ref{finite_lemma}, which yields the required high probability upper bound on $R(h_a)$. Since $R(h_a)=R_a(P_X)$ by definition, the same bound holds for $R_a(P_X)$. The bound also holds for $R_A(P_X)$, since $R_A(P_X):=\underset{X_a \in \mathcal{X}_A}{\min}R_a(P_X)$.
\end{proof}

\subsection{Proof of Lemma \ref{clusteR_FL_theorem}}
\begin{proof}
By the construction of $f_L$, for all points $x$ such that $f(x)$ is not the zero vector, $L_C(x,f):=\underset{\{x': f(x)=f(x')\}}{\max}d_\gamma(x,x')=0$. These points $x$ are precisely those points for which $L_a(x)=0$. Therefore if $R_a(P_X)\leq\epsilon_A$, then $R_C(f)\leq\epsilon_A$.
\end{proof}

\subsection{Proof of Lemma \ref{cluster_hL_theorem}}

\begin{proof}
Applying $f$ to $P_X$ induces a distribution supported at $k+1$ points.   With probability at least $1-\frac{\delta}{3}$, $\mathbb{E}_{x \sim P_X}[\underset{\{f(x'),y'\} \in S_l^Z}{\min}\mathbf{1}(f(x)\neq f(x'))]\leq\underset{t \in [0,1]}{\max}(k+1-\frac{\delta}{3}(1-t)^{-m_l})t=:\alpha$, by Lemma \ref{alpha_lemma}.

Recall that by the definition of $R_C(f)$, if we randomly draw a point $x$ from $P_X$ then with probability at most $\epsilon_C$ there exists some $x'$ such that $d_\gamma(x,x')>0$ and $f(x)=f(x')$. Combining the last two bounds by the union bound, with probability at most $\alpha+\epsilon_C$ there is either no point $\{f(x'),y'\} \in S_l^Z$ such that $f(x)=f(x')$, or $d_\gamma(x,x')>0$ for at least one point $\{f(x'),y'\} \in S_l^Z$ such that $f(x)=f(x')$. If neither of these possibilities occur, $h^Z \circ f$ will be able to correctly classify the point. This is because there is at least one training set point such that $f(x)=f(x')$, for all points in the training set such that $f(x)=f(x')$ we have $d_\gamma(x,x')=0$, we know that if $d_\gamma(x,x')=0$ then the labels of $x$ and $x'$ must agree because $R_B(P_{XY})=0$, and finally we know that $h_Z \circ f$ will achieve zero training set error since a linear classifier in $k$ dimensions can perfectly classify $k+1$ points. Therefore we have with probability at least $1-\frac{\delta}{3}$, $R(h^Z \circ f)\leq \epsilon_C + \alpha$.
\end{proof}

\vspace{-10pt}
\subsection{Proof of Lemma \ref{cluster_PX_hL_lemma}}
\label{cluster_PX_hL_lemma_proof}

\begin{proof}
Run the property testing algorithm described in Algorithm \ref{cluster_Algorithm} to obtain $k=|X_a|$.
Run $\frac{1}{2}k(k-1)$ tests of the following kind, for each pair $X_i, X_j \in X_a$, $i \neq j$. Construct the labeled set $S_u^{ij}$ which includes only those points in $S_u$ that lie in $X_i \cup X_j$, adding the labels $y=1$ for $x \in X_i$ and $y=0$ for $x \in X_j$. Run empirical risk minimization on $S_u^{ij}$ to produce the hypothesis $h^{ij}$, which we assume has the minimum empirical risk of any hypothesis in $H$. Let $\mathcal{S}_u$ be the set of all possible labeled samples constructed by adding labels to $S_u$, subject to the constraints  $R_B(P_{XY})=0$ and $R_E(P_{XY})\geq\epsilon_E$. If these constraints hold, we have for all $h \in H$:

$\beta$

$:=\underset{i,j}{\min}\frac{1}{m_u}\sum\limits_{\{x,y\} \in S_u^{ij}}L(h^{ij}(x),y)$

$\leq\underset{i,j}{\min}\frac{1}{m_u}\sum\limits_{\{x,y\} \in S_u^{ij}}L(h(x),y)$

$\leq\underset{S_u^{*} \in \mathcal{S}_u}{\min}\frac{1}{m_u}\sum\limits_{\{x,y\} \in S_u^{*}}L(h(x),y)$.

$=\underset{S_u^{*} \in \mathcal{S}_u}{\min}\hat{R}(h)$.

The first inequality holds because $h_{ij}$ was constructed through empirical risk minimization guaranteed to find the hypothesis in $H$ with the minimum empirical risk over $S_u^{ij}$. The second inequality holds since $R_B(P_{XY})=0$ implies that the labels of points within clusters agree and $R_E(P_{XY})\geq\epsilon_E>0$ implies that at least one pair of clusters must have different labels. The third equality holds by the definition of empirical risk $\hat{R}(h)$ for some sample $S_u^{*}$.

Using this inequality and the high probability lower bound on $R(h)$ in terms of $\hat{R}(h)$ provided by Lemma \ref{infinite_lemma}, noting that the VC dimension of the class of linear classifiers in $n$ dimensions is $n+1$, yields the required result.
\end{proof}

\section{Example: Manifold representation}
\label{manifold}

We present an example which exploits low dimensional manifold structure present in the unlabeled distribution to learn a new representation, as shown in Figure \ref{manifold_Assumption_fig}. A toy setting for this example is the binary classification problem of determining whether there are crocodiles in a particular section of a river. Crocodiles tend to concentrate in particular regions of the river, since they swim along it and cannot move over land. A representation parametrized by river position rather than latitude and longitude will allow a nearest-neighbor algorithm to better learn where there are crocodiles.

A probabilistic upper bound on the risk of a subsequent supervised learner using the learned representation is shown in Theorem \ref{manifold_theorem}, which instantiates Theorem \ref{sample_complexity_theorem}. This result will be meaningful when the bound is less than 1 (see Figure \ref{example_results_fig} for an example). In this case the learner is the 1-nearest neighbor algorithm using Euclidean distance. Our approach is similar to density-based distances \citep{bijral_semi-supervised_2012}, except that we use the density to learn an explicit representation and continue to use Euclidean distance in the learned feature space.

The parameters used for the manifold example in Section \ref{manifold} are summarized in Table \ref{manifold_table}. Proofs for Theorem \ref{manifold_theorem} and the lemmas used are shown subsequently.  For a given set of parameters, we plot the upper bound on risk using the learned features in Figure \ref{example_results_fig} (left), which depends both on the number of unlabeled points $m_u$ and the number of labeled points $m_l$.

\begin{table}[t]
\small
\begin{center}
\begin{tabular}{p{1cm}|p{2cm}|p{5.9cm}|p{3.5cm}}

\bf Condition & \bf Description & \bf Risk function $R$ & \bf Threshold $\epsilon$ \\ \hline
\multicolumn{4}{l}{\it Upper bound on risk using representation learning} \\ \hline
$A(P_X)$ & If property test passes, marginal distribution is concentrated on one-dimensional manifold &
 $R_a(P_X)=\mathbb{E}_{x \sim P_X}[L_a(x)]$, where $L_a(x)=\mathbf{1}(\underset{x' \in X_a}{\min}(||x-x'||_2)>\frac{s\sqrt{n}}{2})$.  $\hat{R}_a(P_X)$ is the result of Algorithm \ref{manifold_Algorithm} and $R_A(P_X)=\underset{X_a \in \mathcal{X}_A}{\min}R_a(P_X)$. See the proof for the definition of $\mathcal{X}_A$.& $\hat{\epsilon}_A=0$ \newline $\epsilon_A=\frac{1}{m_u}(\frac{n\gamma}{s}\log 3-n\log s + \log\frac{3}{\delta})$\\ \hline
$B(P_{XY})$ & Nearby points measured by distance along the manifold are likely to share labels &$R_B(P_{XY})=\mathbb{E}_{\{x,y\} \sim P_{XY}}[\underset{\{x',y'\}}{\max}L_B(y,y')]$, such that $p(x',y')>0 \wedge d_{r,\sqrt{n}s}(x,x')\leq(j+1)\sqrt{n}s$ and where $L_B(y,y')=\mathbf{1}(y\neq y')$. See Section \ref{manifold_PXY_condition} for definitions of $r$, $j$ and the function $d_{r,\sqrt{n}s}$. & $\epsilon_B\in[0,1]$ is specified by a domain expert \\ \hline
$C(f_L)$ &  Feature learner uses manifold structure to learn new representation  &  $R_C(f)=\mathbb{E}_{x \sim P_X}[\underset{x': |f(x)-f(x')|\leq js}{\max}L_C(x,x')]$, where $L_C(x,x')=\mathbf{1}(d_{r,\sqrt{n}s}(x,x')> (j+1)\sqrt{n}s)$
  & $\epsilon_C=\epsilon_A$\\ \hline
\rule{0pt}{3ex}$D(h_L)$ & 1-nearest neighbor learner can exploit manifold representation & $R(h^Z \circ f)=\mathbb{E}_{\{x,y\} \sim P_{XY}}[L(h^Z(f(x)),y)]$ where $L(y,y')=\mathbf{1}(y\neq y')$ & $\epsilon_{\max}^Z=\epsilon_B+\epsilon_C+\underset{t \in [0,1]}{\max}(\frac{\gamma}{js}-\frac{\delta}{3}(1-t)^{-m_l})t$ \\
\end{tabular}
\end{center}
\caption{Parameters for manifold example} \label{manifold_table}
\end{table}

\begin{figure}[t]
\centering

\begin{tikzpicture}[scale=0.8, transform shape]
\draw [<->] (-3.2,0) -- (3.2,0);
\draw [<->] (0,-3.2) -- (0,3.2);

\draw [<->] (4.3,0) -- (10.7,0);
\draw [-] (4.5,-0.2) -- node[below] {0} (4.5,0);
\draw [-] (10.5,-0.2) -- node[below] {$\gamma$} (10.5,0);
\draw [-,blue] (4.5,0) -- (10.5,0);

\node [above,color=red] at (4.5,0) {$f(x)$};
\path [draw=none,fill=red] (4.4,-0.1) rectangle (4.6,0.1);
\node [below,color=darkspringgreen] at (6.3,0) {$f(x'')$};
\node[fill=darkspringgreen,regular polygon, regular polygon sides=3,inner sep=1.5pt] at (6.3,0) {};
\node [below,color=darkspringgreen] at (9.5,0) {$f(x')$};
\node[fill=darkspringgreen,regular polygon, regular polygon sides=3,inner sep=1.5pt] at (9.5,0) {};

\draw [-,blue] (-0.25,-0.25)-- (0.25,0.25)-- (0.75,0.75)-- (1.25,1.25)-- (1.75,1.25)--(2.25,1.25)-- (2.25,0.75)-- (2.25,0.25)
-- (2.25,-0.25)-- (2.25,-0.75) -- (2.25,-1.25)-- (2.25,-1.75)-- (2.25,-2.25)-- (1.75,-2.25)-- (1.25,-2.25)-- (0.75,-2.25)-- (0.25,-2.25)-- (-0.25,-2.25)-- (-0.75,-2.25)--(-1.25,-2.25)
--(-1.75,-2.25)--(-2.25,-2.25);

\draw[step=0.5,black,thin] (-2.5,-2.5) grid (2.5,2.5);

\path [draw=none,fill=black] (2.4,-2.4) circle (0.1);

\path [draw=none,fill=black] (2.2,-2.2) circle (0.1);

\path [draw=none,fill=black] (2.2,-1.9) circle (0.1);

\path [draw=none,fill=black] (2.3,-1.6) circle (0.1);

\path [draw=none,fill=black] (2.1,-1.4) circle (0.1);

\path [draw=none,fill=black] (2.3,-0.8) circle (0.1);

\path [draw=none,fill=black] (2.1,-0.6) circle (0.1);

\path [draw=none,fill=black] (2.4,-0.2) circle (0.1);

\path [draw=none,fill=black] (2.2,-0.1) circle (0.1);

\path [draw=none,fill=black] (2.1,0.2) circle (0.1);

\path [draw=none,fill=black] (2.4,0.7) circle (0.1);

\path [draw=none,fill=black] (2.4,1.2) circle (0.1);

\path [draw=none,fill=black] (1.6,1.2) circle (0.1);

\path [draw=none,fill=black] (1.1,1.2) circle (0.1);

\path [draw=none,fill=black] (0.7,0.6) circle (0.1);

\path [draw=none,fill=black] (0.3,0.3) circle (0.1);

\path [draw=none,fill=black] (-0.4,-0.3) circle (0.1);

\path [draw=none,fill=black] (0.4,0.1) circle (0.1);

\path [draw=none,fill=black] (0.8,0.8) circle (0.1);

\path [draw=none,fill=black] (0.4,-2.1) circle (0.1);

\path [draw=none,fill=black] (0.7,-2.2) circle (0.1);

\path [draw=none,fill=black] (1.2,-2.3) circle (0.1);

\path [draw=none,fill=black] (1.8,-2.4) circle (0.1);

\path [draw=none,fill=black] (-0.4,-2.1) circle (0.1);

\path [draw=none,fill=black] (-0.6,-2.2) circle (0.1);

\path [draw=none,fill=black] (-0.9,-2.2) circle (0.1);

\path [draw=none,fill=black] (-1.2,-2.2) circle (0.1);

\path [draw=none,fill=black] (-1.4,-2.3) circle (0.1);

\path [draw=none,fill=black] (-1.8,-2.1) circle (0.1);

\path [draw=none,fill=black] (-1.8,-2.4) circle (0.1);

\path [draw=none,fill=black] (-2.1,-2.4) circle (0.1);

\path [draw=none,fill=black] (-2.4,-2.3) circle (0.1);

\node [below,color=red] at (-0.1,-0.4) {$x$};

\path [draw=none,fill=red] (-0.2,-0.5) rectangle (0,-0.3);

\node [above,color=darkspringgreen] at (1.3,1.4) {$x''$};

\node[fill=darkspringgreen,regular polygon, regular polygon sides=3,inner sep=1.5pt] at (1.3,1.4) {};

\node [above,color=darkspringgreen] at (-0.2,-2.1) {$x'$};

\node[fill=darkspringgreen,regular polygon, regular polygon sides=3,inner sep=1.5pt] at (-0.2,-2.1) {};

\draw[->,darkspringgreen] (1.5,1.4) to[out=-5,in=85] (6.3,0.2);
\draw[->,darkspringgreen] (-0.1,-2.2) to[out=-20,in=-85] (9.5,-0.5);
\draw[->,red] (0,-0.6) to[out=-5,in=-145] (4.4,-0.1);

\end{tikzpicture}
\vspace{-20pt}
\caption{Example of a map $f$ (arrows) from the original input space $X=\mathbb{R}^2$ (left) to the feature space $Z=\mathbb{R}$ (right). The manifold (blue line, left) learned from the unlabeled sample (black dots, left)  is used to represent the data in the interval $[0,\gamma]$ (blue line, right). In the original space a 1-nearest-neighbor classifier with Euclidean distance uses the label of training point $x'$ to classify the test point $x$ since $||x-x'||_2<||x-x''||_2$, while in the learned feature space it uses the label of training point $x''$ since $|f(x)-f(x'')|<|f(x)-f(x')|$. With high probability, $x$ and $x''$ share labels (by Assumption \ref{manifold_PXY_Assumption}), using the parameter $j=3$.
} \label{manifold_Assumption_fig}
\vspace{-15pt}
\end{figure}
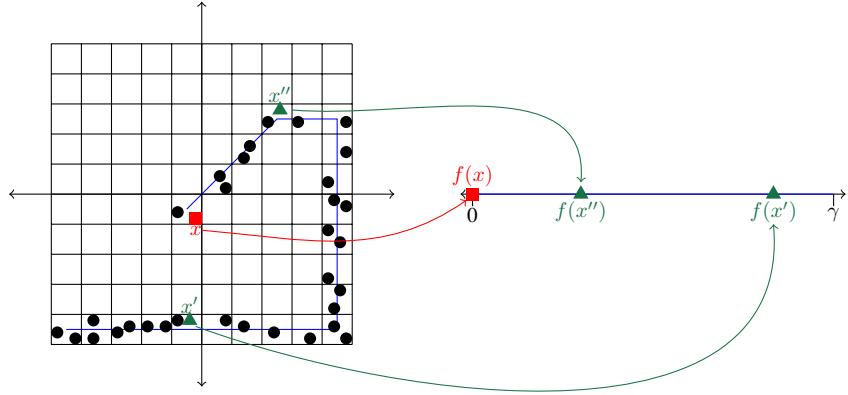

\begin{theorem}
\label{manifold_theorem}
Let $\hat{R}_a(P_X)$ be the result of the manifold property test described in Algorithm \ref{manifold_Algorithm} run on an unlabeled sample $S_u$. Let $s$ be a side length parameter and $\gamma$ be a manifold length parameter (see Section \ref{manifold_PX_condition}). Let $j$ be a proximity parameter and $\epsilon_B$ be a label agreement parameter (see Section \ref{manifold_PXY_condition}). Let $\epsilon_{\max}^Z:=\frac{1}{m_u}(\frac{n\gamma}{s}\log 3-n\log s + \log\frac{3}{\delta})+\epsilon_B+\underset{t \in [0,1]}{\max}(\frac{\gamma}{js}-\frac{\delta}{3}(1-t)^{-m_l})t$. Suppose $\hat{R}_a(P_X)=0$, $P_{XY}$ satisfies Assumption \ref{manifold_PXY_Assumption}, $f_L$ is the feature learner in Section \ref{manifold_fL_Condition}, and $h_L$ is the hypothesis learner in Section \ref{manifold_hL_property}. Then if a hypothesis $h^Z \circ f$ is constructed from $S_u$ and a labeled sample $S_l$, with probability at least $1-\delta$, $R(h^Z \circ f) \leq\epsilon_{\max}^Z$.
\end{theorem}

\begin{proof}
With probability at least $1-\delta$, by the union bound, we have $A(P_X)$ by Lemma \ref{manifold_PX_theorem}, $B(P_{XY})$ by Assumption \ref{manifold_PXY_Assumption}, $C(f_L)$ by Lemma \ref{manifold_fL_lemma} and $D(h_L)$ by Lemma \ref{manifold_hL_lemma}. Applying Theorem \ref{sample_complexity_theorem} we have $R(h^Z \circ f) \leq\epsilon_{\max}^Z$.
\end{proof}

In principle it should also be possible to provide a lower bound on the performance of the subsequent supervised learner using the original inputs. We found that the distribution-independent upper bound presented here, which is tighter with more labeled samples, is not tight enough to be less than such a distribution-specific lower bound, which is tighter with few labeled samples (see Appendix \ref{motivation}). In future our upper bound may be tightened by creating dependence on the distribution $P_X$.

\begin{figure}[t]
\includegraphics[scale=0.23]{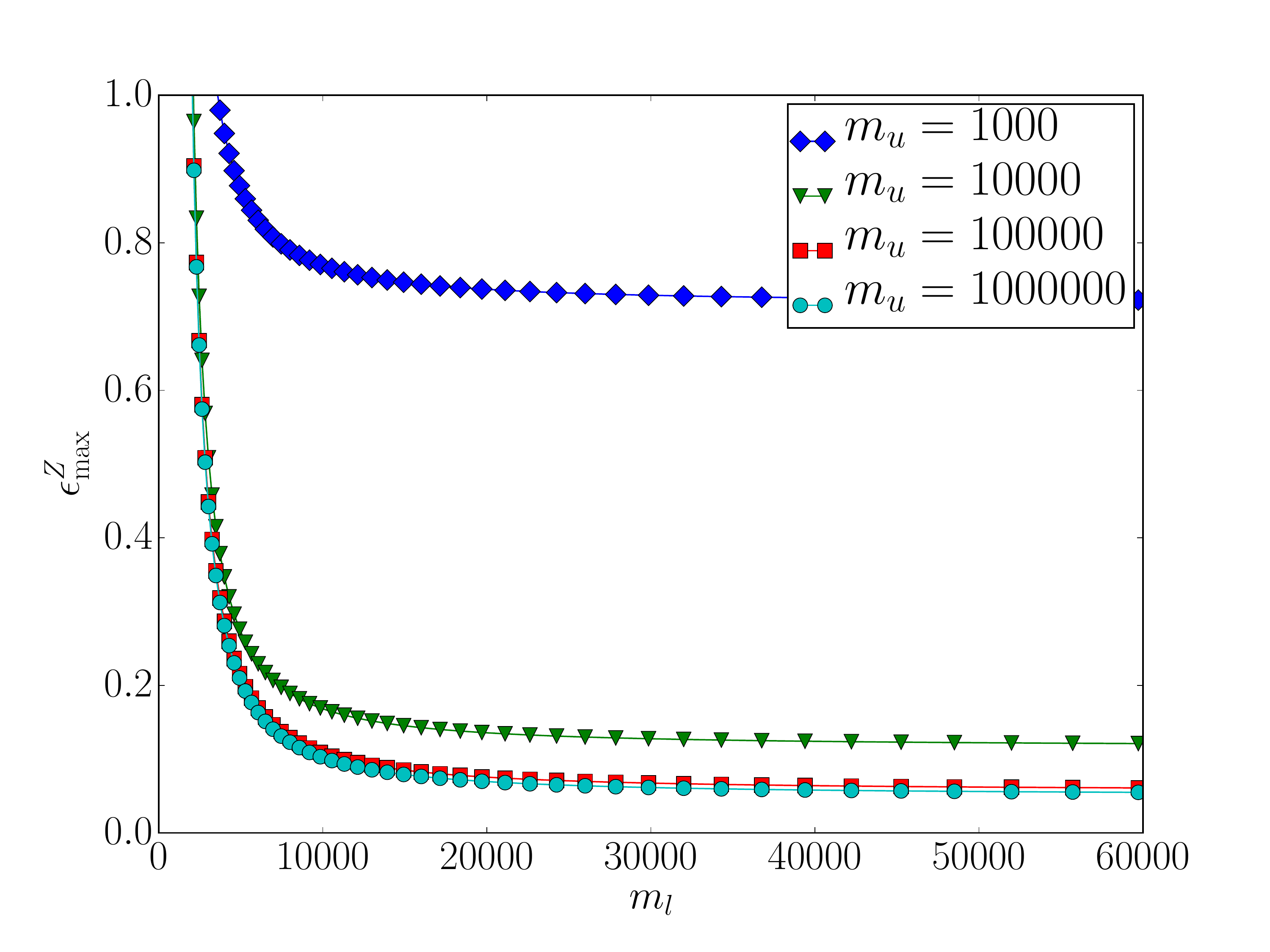}
\includegraphics[scale=0.23]{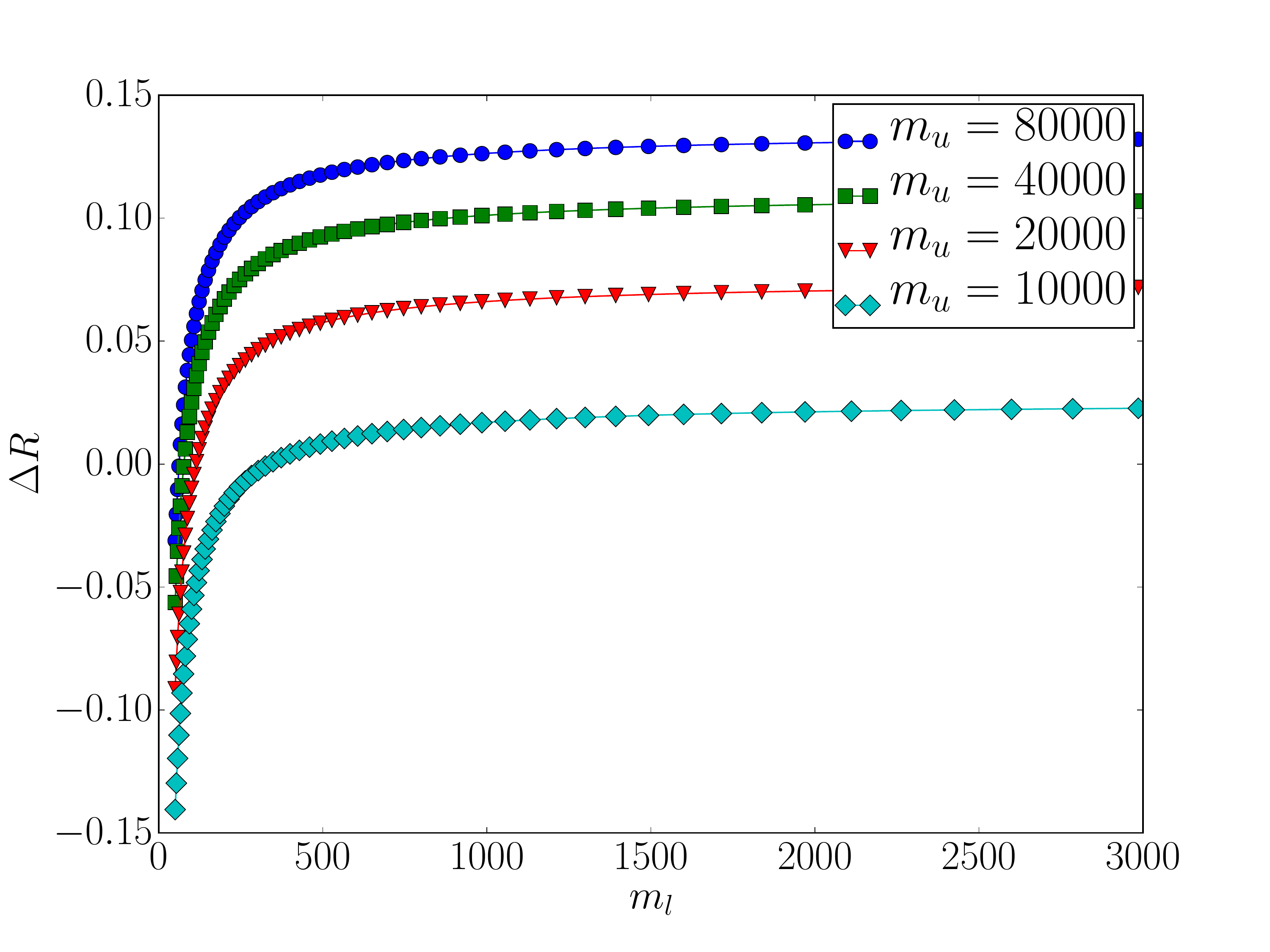}
\caption{Risk upper bound $\epsilon_{\max}^Z$ for manifold example calculated using Theorem \ref{manifold_theorem} (left), where the parameters are $\epsilon_B=0.05$, $j=3$, $s=0.1$, $n=2$, $\gamma=20$  and $\delta=0.05$. Risk gap $\Delta R$ for cluster example calculated using Theorem \ref{cluster_theorem} (right), where the parameters are $\beta=0.2$, $s=0.1$, $n=2$, $k=2$ and $\delta=0.05$.} \label{example_results_fig}
\end{figure}

\subsection{Condition $A(P_X)$}
\label{manifold_PX_condition}

We consider a property which measures the extent to which $P_X$ is concentrated on particular type of one-dimensional manifold, and describe an algorithm for testing this property from an unlabeled sample. The quantity $R_A(P_X) \in [0,1]$, defined below, describes the extent to which the property holds, with $R_A(P_X)=0$ indicating that it perfectly holds. While we analyze a particular class of one-dimensional manifolds, in future we envisage that this restriction can be relaxed.

Given a distribution $P_X$, a hypercube side length parameter $s$ and a manifold length parameter $\gamma$, let $\mathcal{X}_A$ be the set of all regions that form a one-dimensional manifold on which $P_X$ is approximately concentrated (see the proof for a formal definition). For some region $X_a$, let $L_a(x):=\mathbf{1}(\underset{x' \in X_a}{\min}||x-x'||_2>\frac{s\sqrt{n}}{2})$ and $\hat{R}_a(P_X)$ be a quantity calculated by the property test described in Algorithm \ref{manifold_Algorithm}, where if $\hat{R}_a(P_X)=0$ then $L_a(x)=0$ for all $x \in S_u$. Let $R_a(P_X):=\mathbb{E}_{x \sim P_X}[L_a(x)]$ and $R_A(P_X):=\underset{X_a \in \mathcal{X}_A}{\min} R_a(P_X)$.

\begin{lemma}
\label{manifold_PX_theorem}
Let $S_u$ be a sample drawn from $P_X$ and let $\hat{R}_a(P_X)$ be calculated using $S_u$ and the property test described in Algorithm \ref{manifold_Algorithm}. Let $\hat{\epsilon}_A:=0$ and $\epsilon_A:=\frac{1}{m_u}(\frac{n\gamma}{s}\log 3-n\log s + \log\frac{3}{\delta})$. With probability at least $1-\frac{\delta}{3}$, $\hat{R}_a(P_X)\leq\hat{\epsilon}_A \implies R_a(P_X)\leq \epsilon_A \wedge R_A(P_X)\leq\epsilon_A$.
\end{lemma}

First, given a distribution $P_X$, a hypercube side length parameter $s$ and a manifold length parameter $\gamma$, we define $\mathcal{X}_A$ be the set of all regions $X_a \subseteq X$ such that:
\begin{enumerate}
\item $X_a$ is a connected, non-self-intersecting curve of length at most $\gamma$
\item For all $x \in X_a$, there exists some point $x'$ supported by $P_X$ such that $||x-x'||_2<\frac{s\sqrt{n}}{2}$.
\end{enumerate}

The algorithm for testing for and learning a one-dimensional manifold is presented in Algorithm \ref{manifold_Algorithm}. The algorithm adopts a depth-first search strategy and returns failure if it cannot find a one-dimensional manifold that all of $S_u$ lies near. In future this restriction could be relaxed to find a manifold that \textit{most} of the data lies near, allowing us to consider the case where $\hat{\epsilon}_A>0$.

\begin{algorithm}[h!]

 \caption{Testing for and learning a one-dimensional manifold}
 \label{manifold_Algorithm}
 \KwData{unlabeled sample $S_u$, input space $X$, side length $s$, manifold length $\gamma$}
 \KwResult{Curve ($X_a$) and proportion of $S_u$ not concentrated on $X_a$ ($\hat{R}_a$), or failure}
 \textit{Divide $X$ into hypercubes of side $s$} \\
 \For {each hypercube c}{
   \If{$c$ is not empty}{
   \textit{path}$:=$\textit{explore(c,0)}\\
   \If{path is not failure}{
      $X_a:=$ \textit{curve connecting the centers of the hypercubes in the order specified by path}
      $\hat{R}_a:=0$ \\
      \Return $\{X_a,\hat{R}_a\}$
   }
  }
 }
\Return \textit{failure} \\~\\
\SetKwProg{myproc}{Function}{}{}
\myproc{explore(path,path\_distance)}{

	\If{path\_distance$>\gamma$}{
		\Return \textit{failure}
	}
	\For{$i=$1:length(path)-n}{
		\If {path[length(path)]=path[i]}{
			\Return \textit{failure}
		}
	}
	\If{all squares not on path are empty}{
		\Return \textit{path}
	}
	\For {each non-empty neighbor of path[length(path)] not currently on path}{
		\textit{new\_path}$:=$[\textit{path},\textit{neighbor}]\\
		\textit{distance\_added}$:=$\textit{distance from path[length(path)] to center of neighbor}\\
		\textit{new\_path\_distance}$:=$\textit{path\_distance+distance\_added}\\
		\If{explore(\textit{new\_path}) is not failure}{
			\Return \textit{explore(new\_path,new\_path\_distance)}
		}
	}

 }
 \Return \textit{failure}
\end{algorithm}

We now state the proof of Lemma \ref{manifold_PX_theorem}.

\begin{proof}
If $\hat{R}_a(P_X)=0$, then Algorithm \ref{manifold_Algorithm} succeeds and also returns some $X_a \in \mathcal{X}_A$.
For each region $X_a \in \mathcal{X}_A$ define a hypothesis $h_a$, where $h_a(x)=1$ if $\mathbf{1}(\underset{x' \in X_a}{\min}||x-x'||_2\leq\frac{s\sqrt{n}}{2})$, and $h_a(x)=0$ otherwise. Let $H_A$ be the set of such hypotheses, where $|H_A|\leq s^{-n}(3^n)^{\frac{\gamma}{s}}$, since the path may start at any of the $s^{-n}$ hypercubes and then may take at most $\frac{\gamma}{s}$ steps, each of which must be to one of at most $3^n$ options including neighboring hypercubes or remaining at the same hypercube if the path length is less than $\gamma$. Let $R(h_a):=\mathbb{E}_{x \sim P_X}[L(h_a(x),1)]$ and $\hat{R}(h_a):=\frac{1}{m_u}\sum\limits_{x \in S_u}L(h_a(x),1)$, recalling that $L$ is 0/1 loss.

Because $\hat{R}(h_a)=\frac{1}{m_u}\sum\limits_{x \in S_u}L_a(x)$ by definition, and $\frac{1}{m_u}\sum\limits_{x \in S_u}L_a(x)=0$ if $\hat{R}_a(P_X)=0$ by the construction of Algorithm \ref{manifold_Algorithm}, we have $\hat{R}(h_a)=0$. Hence we may apply Lemma \ref{finite_lemma}, which yields the required high probability upper bound on $R(h_a)$. Since $R(h_a)=R_a(P_X)$ by definition, the same bound holds for $R_a(P_X)$. The bound also holds for $R_A(P_X)$, since $R_A(P_X):=\underset{X_a \in \mathcal{X}_A}{\min}R_a(P_X)$.
\end{proof}

\vspace{-15pt}
\subsection{Condition $B(P_{XY})$}
\label{manifold_PXY_condition}

Given that the data lies on a low dimensional manifold, we assume that close points as measured by distance along the manifold are likely to share labels. The discovery of manifold structure alone may not be useful for prediction, unless the joint distribution also exhibits shared structure.

Let $\mathcal{C}$ be the set of all hypercubes containing at least one point in $S_u$, $p_c$ be the probability mass within hypercube $c$ and $\hat{p}_c$ be the empirical estimate of this mass obtained from the sample $S_u$. Let $r \leq\underset{c \in \mathcal{C}}{\min}(\min\{p_c | P[Bin(m_u;p_c) \geq \hat{p}_cm_u]\geq \frac{s \delta }{3\gamma}\})$ be a parameter indicating the minimum probability mass contained in small regions along the manifold. Let $d_{r,\sqrt{n}s}(x,x')$ be the length of the shortest path between $x$ and $x'$ such that for all points $x''$ along the path the following condition holds: there exists a region $T$ such that $\underset{x_T \in T}{\max}||x''-x_T||_2\leq\sqrt{n}s$ and $\int_Tp(x_T)dx_T\geq r$. Let $j$ be a parameter controlling how close points must be such that they are likely to share labels.

Let $R_B(P_{XY}):=\mathbb{E}_{\{x,y\} \sim P_{XY}}[\underset{\{x',y'\}}{\max}L_B(y,y')]$, where the maximization is over points $\{x',y'\}: p(x',y')>0 \wedge d_{r,\sqrt{n}s}(x,x')\leq(j+1)\sqrt{n}s$, and $L_B(y,y'):=\mathbf{1}(y\neq y')$. A small value of $R_B(P_{XY})$ indicates that for most points $x$, all points $x'$ that are close in terms of manifold distance in the sense that $d_{r,\sqrt{n}s}(x,x')\leq (j+1)\sqrt{n}s$, will have the same label as $x$.

\begin{assumption}
\label{manifold_PXY_Assumption}
Let $\epsilon_B\in[0,1]$ be specified by a domain expert. Assume that $R_A(P_X)\leq\epsilon_A \implies R_B(P_{XY})\leq\epsilon_B$.
\end{assumption}

\subsection{Condition $C(f_L)$}
\label{manifold_fL_Condition}

We specify a feature learner $f_L$  which, if one-dimensional manifold structure can be detected, exploits this structure to learn a representation. Define $f_L$ as follows, yielding the feature map $f$.  Run the property test described in Algorithm \ref{manifold_Algorithm}, which we assume passes and returns the curve $X_a$. Let $x_0$ be the center of the first hypercube on the curve and set $f(x)=0$ for all points in this hypercube. For points $x$ lying in other hypercubes through which $X_a$ passes, set $f(x)$ to be the distance along $X_a$ from $x_0$ to the center of the point's hypercube. For points $x$ lying in hypercubes whose centers are not in $X_a$, set $f(x)$ to be some constant such that $f(x)\ll -\gamma$.

We would now like to quantify the probability that if $f(x)$ and $f(x')$ are close, points $x$ and $x'$ are close in the original space as measured by the shortest distance of a path between them through regions of high probability density. Let $R_C(f):=\mathbb{E}_{x \sim P_X}[\underset{x': |f(x)-f(x')|\leq js}{\max}L_C(x,x')]$ and $L_C(x,x'):=\mathbf{1}(d_{r,\sqrt{n}s}(x,x')> (j+1)\sqrt{n}s)$. Note that the definition of $R_C(f)$ depends on the parameter $r$, which is upper bounded by an expression dependent on $\delta$.

\begin{lemma}
\label{manifold_fL_lemma}
Let $\epsilon_C:=\epsilon_A$. With probability at least $1-\frac{\delta}{3}$, $R_a(P_X)\leq\epsilon_A \implies R_C(f)\leq\epsilon_C$.
\end{lemma}

\begin{proof}
Recall that $p_c$ is the probability mass within hypercube $c$ and $\hat{p}_c$ is the empirical estimate of this mass obtained from the sample $S_u$. With probability at least $1-\frac{s \delta }{3\gamma}$ we have the following for a single hypercube $c$ whose center lies on $X_a$:

$p_c$

$\geq\min\{p_c | P[Bin(m_u;p_c) \geq \hat{p}_cm_u]\geq \frac{s \delta }{3\gamma}\}$





$\geq r$.

There are at most $\frac{\gamma}{s}$ hypercubes whose centers lie on $X_a$. By the union bound, with probability at least $1-\frac{\delta}{3}$, $p_c\geq r$ for all of these hypercubes. In that case, if $x$ is in a hypercube whose center lies on $X_a$, then for all $x'$, if $|f(x)-f(x')|\leq js$ then there must be a path between $x$ and $x'$ passing through at most $j+1$ hypercubes each of which contains probability mass at least $r$. Such a path must be of length at most $(j+1)\sqrt{n}s$. Therefore for such values of $x$, for all $x'$, $|f(x)-f(x')|\leq js\implies d_{r,\sqrt{n}s}(x,x')\leq (j+1)\sqrt{n}s$ and and hence $L_C(x,x')=0$. Since $R_A(P_X)\leq\epsilon_A$, the chances of drawing some $x$ in a hypercube whose center does not lie on $X_a$ is at most $\epsilon_A$. Therefore $R_C(f)\leq\epsilon_A$, still with probability at least $1-\frac{\delta}{3}$.
\end{proof}

\subsection{Condition $D(h_L)$}
\label{manifold_hL_property}

Let $h_L$ be the 1-nearest neighbor learner using Euclidean distance trained on the labeled sample $S_l$, yielding the hypothesis $h(x)=y^*$, where $\{x^*,y^*\}=\underset{\{x',y'\} \in S_l}{\mathrm{argmin}}||x-x'||_2$. Similarly, let $h^Z_L$ be the 1-nearest neighbor learner using Euclidean distance trained on the transformed labeled sample $S_l^Z$, yielding the hypothesis $h^Z(f(x))=y^*$, where $\{f(x^*),y^*\}=\underset{\{f(x'),y'\} \in S_l^Z}{\mathrm{argmin}}|f(x)-f(x')|$. Note that the bound on $R(h^Z \circ f)$ shown is independent of $P_{XY}$ given $R_B(P_{XY})\leq\epsilon_B$ and $R_C(f)\leq\epsilon_C$.

\begin{lemma}
\label{manifold_hL_lemma}
Let $\epsilon_{\max}^Z:=\underset{t \in [0,1]}{\max}(\frac{\gamma}{js}-\frac{\delta}{3}(1-t)^{-m_l})t+\epsilon_B+\epsilon_C$. With probability at least $1-\frac{\delta}{3}$, $R_B(P_{XY})\leq\epsilon_B \wedge R_C(f)\leq\epsilon_C \implies R(h^Z\circ f)\leq\epsilon_{\max}^Z$.
\end{lemma}

\begin{proof}
Partition $[0,\gamma]$ into $\frac{\gamma}{js}$ intervals of equal width. For a point $f(x)$ in an interval containing at least one point in the sample $S_l^Z$, $\underset{\{f(x'),y'\} \in S_l^Z}{\min}|f(x)-f(x')|\leq js$. With probability at least $1-\frac{\delta}{3}$, $\mathbb{E}_{x \sim P_X}[\mathbf{1}(\underset{\{f(x'),y'\} \in S_l^Z}{\min}|f(x)-f(x')|> js)]\leq\underset{t \in [0,1]}{\max}(\frac{\gamma}{js}-\frac{\delta}{3}(1-t)^{-m_l})t=:\alpha$, by Lemma \ref{alpha_lemma}.

Recall that by the definition of $R_C(f)$, if we randomly draw a point $x$ from $P_X$, the probability that there exists some $x'$ such that $|f(x)-f(x')|\leq js$ and $d_{r,\sqrt{n}s}(x,x')> (j+1)\sqrt{n}s$ is at most $\epsilon_C$. Furthermore recall by the definition of $R_B(P_{XY})$, if we randomly draw a point $\{x,y\}$ from $P_{XY}$, the probability that there exists some supported $\{x',y'\}$ such that $d_{r,\sqrt{n}s}(x,x')\leq (j+1)\sqrt{n}s$  and $y' \neq y$ is at most $\epsilon_B$. Combining the last three bounds by the union bound, with probability at most $\alpha+\epsilon_C + \epsilon_B$, for a point $\{x,y\}$ randomly drawn from $P_{XY}$, there is either no point $\{f(x'),y'\} \in S_l^Z$ such that $|f(x)-f(x')|\leq js$, or $d_{r,\sqrt{n}s}(x,x')> (j+1)\sqrt{n}s$ for at least one point $\{f(x'),y'\} \in S_l^Z$ such that $|f(x)-f(x')|\leq js$, or $y' \neq y$ for at least one point in $\{f(x'),y'\} \in S_l^Z$ such that $d_{r,\sqrt{n}s}(x,x')\leq (j+1)\sqrt{n}s$. If none of these possibilities occur, a hypothesis $h^Z \circ f$ constructed using the 1-nearest neighbor learner and $S_l^Z$ will be able to correctly classify $x$. Therefore we have with probability at least $1-\frac{\delta}{3}$, $R(h^Z \circ f)\leq \alpha + \epsilon_B + \epsilon_C$ as required.
\end{proof}

We have now stated all lemmas used by the proof of Theorem \ref{manifold_theorem}, which is the main result for the manifold representation example.

\end{document}